\documentclass[letterpaper,11pt]{article}
\usepackage{times}
\usepackage[margin=1in]{geometry}
\usepackage{latexsym, amsthm, amscd, amsfonts, mathrsfs, amsmath, amssymb, stmaryrd, tikz-cd, mathrsfs, bbm, url, esint, mathtools, enumerate}

\usepackage{dsfont}

\newcommand{\Cov}{\mathrm{cov}}
\newcommand{\pr}{\mathbb{P}}
\newcommand{\E}{\mathbb{E}}

\DeclareMathOperator{\argmin}{argmin}

\DeclareMathOperator{\Ramp}{Ramp}
\DeclareMathOperator{\ReLU}{ReLU}
\DeclareMathOperator{\Unif}{Unif}
\DeclareMathOperator{\Rad}{Rad}

\DeclareMathOperator{\NN}{NN}

\DeclareMathOperator{\poly}{poly}
\DeclareMathOperator{\TV}{TV}
\newcommand{\cF}{\mathcal{F}}
\newcommand{\cX}{\mathcal{X}}

\newcommand{\cD}{\mathcal{D}}
\newcommand{\cN}{\mathcal{N}}
\newcommand{\cP}{\mathcal{P}}
\newcommand{\bR}{\mathbb{R}}

\newcommand{\bN}{\mathbb{N}}

\newcommand{\cL}{\mathcal{L}}
\newcommand{\cB}{\mathcal{B}}

\theoremstyle{plain}
\newtheorem{definition}{Definition}
\theoremstyle{plain}
\newtheorem{theorem}{Theorem}
\theoremstyle{plain}
\theoremstyle{plain}
\theoremstyle{plain}
\theoremstyle{plain}
\theoremstyle{plain}
\theoremstyle{plain}
\newtheorem{lemma}{Lemma}
\theoremstyle{plain}
\newtheorem{corollary}{Corollary}
\theoremstyle{plain}
\theoremstyle{remark}
\newtheorem{remark}{Remark}
\theoremstyle{remark}
\theoremstyle{plain}

\begin{document}

\author{
Elisabetta Cornacchia\thanks{Institute of Mathematics, École Polytechnique Fédérale de Lausanne. Email: \textup{\tt elisabetta.cornacchia@epfl.ch}}
\and
Elchanan Mossel\thanks{Department of Mathematics and IDSS, Massachusetts Institute of Technology. Email: \textup{\tt elmos@mit.edu}}}

\title{A Mathematical Model for Curriculum Learning for Parities}
\date{}

\maketitle

\begin{abstract}
\textit{Curriculum learning (CL)} - training using samples that are generated and presented in a meaningful order - was introduced in the machine learning context around a decade ago. 
While CL has been extensively used and analysed empirically, there has been very little mathematical justification for its advantages. 
    We introduce a CL model for learning the class of $k$-parities on $d$ bits of a binary string with a neural network trained by stochastic gradient descent (SGD). We show that a wise choice of training examples involving two or more product distributions, allows to reduce significantly the computational cost of learning this class of functions, compared to learning under the uniform distribution. 
    Furthermore, we show that for another class of functions - namely the `Hamming mixtures' - CL strategies involving a bounded number of product distributions are not beneficial.  
\end{abstract}

\section{Introduction}
Several experimental studies have shown that humans and animals learn considerably better if the learning materials are presented in a curated, rather than random, order~\cite{elio1984effects,ross1990generalizing,avrahami1997teaching,shafto2014rational}. This is broadly reflected in the educational system of our society, where learning is guided by an highly organized curriculum. This may involve several learning steps: with easy concepts introduced at first and harder concepts built from previous stages.

Inspired by this, \cite{bengio2009curriculum} formalized a \textit{curriculum learning (CL)} paradigm in the context of machine learning and showed that for various learning tasks it provided improvements in both the training speed and the performance obtained at convergence. This seminal paper inspired many subsequent works, that studied curriculum learning strategies in various application domains, e.g. computer vision~\cite{sarafianos2017curriculum,dong2017multi}, computational biology~\cite{xiong2021modeling}, auto-ML~\cite{graves2017automated}, natural language modelling~\cite{shi2013k,zaremba2014learning,shi2015recurrent,campos2021curriculum}. 
While extensive empirical analysis of CL strategies have been carried out, there is a lack of theoretical analysis. In this paper, we make progress in this direction. 

A stylized family of functions that is known to pose computational barriers is the class of $k$-parities over $d$ bits of a binary string. In this work we focus on this class. 
To define this class: for each subset $S$ of coordinates, the parity over $S$ is defined as $+1$ if the number of negative bits in $S$ is even, and $-1$ otherwise, i.e. $\chi_S(x):= \prod_{i \in S} x_i$, $x_i \in \{\pm 1 \}$. The class of $k$-parities contains all $\chi_S $ such that $|S| =k$ and it has cardinality ${ d \choose k}$. Learning $k$-parities requires learning the support of $\chi_S$ by observing samples $(x, \chi_S(x))$, $x \in \{ \pm 1 \}^d$, with the knowledge of the cardinality of $S$ being $k$. This requires finding the right target function among the ${ d \choose k}$ functions belonging to the class. 

Learning parities is always possible, and efficiently so, by specialized methods (e.g. Gaussian elimination over the field of two elements). Moreover,~(\cite{AS20}) showed that there exists a neural net that learns parities of any degree if trained by SGD with small batch size. However, this is a rather unconventional net. In fact, under the uniform distribution, parities are not efficiently learnable by population queries with any polynomially small noise. 
The latter can be explained as follows. Assume we sample our binary string uniformly at random, i.e. for each $i \in \{ 1,..., d \}$, $x_i \sim \Rad(1/2)$\footnote{$z \sim \Rad(p)$ if $\pr (z=1) = 1- \pr (z=-1) = p $.}. Then, the covariance between two parities $\chi_S, \chi_{S'}$ is given by: 
\begin{align*}
\E_{x \sim \Rad(1/2)^{\otimes d} }\left[ \chi_S(x) \chi_{S'} (x) \right] = \begin{cases} 
& 1 \quad \text{ if } S =S',\\
& 0 \quad \text{ if  } S \neq S',
\end{cases}
\end{align*}
where $x \sim \Rad(1/2)^{\otimes d}$ denotes the product measure such that $x_i \overset{iid}{\sim} \Rad(1/2)$, $i \in \{1,...,d\}$. More abstractly, a parity function of $k$ bits is uncorrelated with {\em any} function of $k-1$ or less bits. 
This property makes parities hard to learn for any progressive algorithm, such as gradient descent. 
Indeed, when trying to learn the set of relevant features, a learner cannot know how close its progressive guesses are to the true set. In other words, all wrong guesses are indistinguishable, which suggests that the learner might have to perform exhaustive search among all the ${d \choose k}$ sets. 

The hardness of learning unbiased parities - and more in general any classes of functions with low cross-correlations - with gradient descent has been analysed e.g. in~\cite{AS20}, where the authors show a lower bound on the computational complexity of learning low cross-correlated classes with
gradient-based algorithms with bounded gradient precision.
For $k$-parities, this gives a computational lower bound of $d^{\Omega(k) }$ for any architecture and initialization.


However, if we look at different product distributions, then the inner product of a monomial and a component $x_i$ that is inside and outside the support becomes distinguishable.
Suppose the inputs are generated as $x \sim \Rad(p)^{\otimes d}$, for some $p \in (0,1)$. Then the covariance between $\chi_S$ and $\chi_{S'}$ is: 
\begin{align*}
        \E_{x \sim \Rad(p)^{\otimes d}} &\Big[ (\chi_S(x) - \E[\chi_S(x)])\cdot  (\chi_{S'}(x) - \E[\chi_{S'}(x)])\Big]\\
        &= \mu_p^{2k - | S \cap S'|} - \mu_p^{2k},
\end{align*}
where we denoted by $ \mu_p : = \E_{z \sim \Rad(p) } [z] = 2p-1$. 
This implies that if for instance $|p-0.5| > 0.1$, just computing correlations with each bit, will recover the parity with complexity linear in $d$ and exponential in $k$.
If we choose $p = 1-1/k$, say, we can get a complexity that is linear in $d$ and polynomial in $k$. Moreover, the statements above hold even for parities with random noise. 

This may lead one to believe that learning biased parities is easy for gradient descent based methods for deep nets. Indeed,~\cite{malach2021quantifying} showed that biased parities are learnable by SGD on a differentiable model consisting of a linear predictor and a fixed module implementing the parity.
However, if we consider fully connected networks, as our experiments show (Figure~\ref{fig:curriculum_gen}), while gradient descent for a $p$ far from a half converges efficiently to zero training loss, the learned function actually has {\em non-negligible error} when computed with respect to the uniform measure. This is intuitively related to the fact that, by concentration of measure, there are essentially no examples with Hamming weight\footnote{The Hamming weight of $x \in \{ \pm 1 \}^d $ is: $H(x) = \sum_{i=1}^d \mathds{1}(x_i =1)$.} close to $d/2$ in the training set sampled under $\Rad(p)^{\otimes d}$, and therefore it is not reasonable to expect for a general algorithm like gradient descent on fully connected networks (that does not know that the target function is a parity) to learn the value of the function on such inputs.

We thus propose a more subtle question: 
Is it possible to generate examples from different product distributions and present them in a specific order, in such a way that the error with respect to the unbiased measure becomes negligible? 

As we mentioned, training on examples sampled from a biased measure is not sufficient to learn the parity under the unbiased measure. However, it does identify the support of the parity. Our curriculum learning strategy is the following: We initially train on inputs sampled from $\Rad(p)^{\otimes d}$ with $p$ close to $1$, then we move (either gradually or by sharp steps) towards the unbiased distribution $ \Rad(1/2)^{\otimes d}$. We show that this strategy allows to learn the $k$-parity problem with a computational cost of $d^{O(1)}$ with SGD on the hinge loss or on the \textit{covariance loss} (see Def.~\ref{def:covariance_loss}). In our proof, we consider layer-wise training (similarly to e.g.~\cite{malach2021quantifying,malach2020computational,barak2022hidden}) and the result is valid for any (even) $k$ and $d$. 

As we mentioned earlier, the failure of learning parities under the uniform distribution from samples coming from a different product measure is due to concentration of Hamming weight.
This leads us to consider a family of functions that we call \textit{Hamming mixtures}. Given an input $x$, the output of a Hamming mixture is a parity of a subset $S$ of the coordinates, where the subset $S$ depends on the Hamming weight of $x$ (see Def.~\ref{def:Hamming_mixture}). Our intuition is based on the fact that given a polynomial number of samples from, say, the $p=1/4$ biased measure, it is impossible to distinguish between a certain parity $\chi_S$ and a function that is $\chi_S$, for $x$'s whose Hamming weight is at most $3/8 d $, and a different function $\chi_T$, for $x$'s 
whose Hamming weight is more than $3/8 d$, for some $T$ that is disjoint from $S$. In other words, a general algorithm does not know whether there is consistency between $x$'s with different Hamming weight. We show a lower bound for learning Hamming mixtures with curriculum strategies that do not allow to get enough samples with relevant Hamming weight.

Of course, curriculum learning strategies with enough learning steps allow to obtain samples from several product distributions, and thus with all relevant Hamming weights. Therefore, we expect that CL strategies with unboundedly many learning steps will be able to learn the Hamming mixtures.

While our results are restricted to a limited and stylized setting, we believe they may open new research directions. Indeed, we believe that our general idea of introducing correlation among subsets of the input coordinates to facilitate learning, may apply to more general settings. 
We discuss some of these future directions in the conclusion section of the paper. 

Importantly, we remark that a limitation of the curriculum strategy presented in this paper is that it requires an oracle that provides labeled samples from arbitrary product measures. However, in applications one usually has a fixed dataset and would like to select samples in a suitable order, to facilitate learning. We leave to future work the analysis of a setting where curriculum and non-curriculum have a common sampling distribution.





\paragraph{Contributions.} Our contributions are the following.
\begin{enumerate}
    \item We propose and formalize a mathematical model for curriculum learning; 
    \item We prove that our curriculum strategy allows to learn $k$-parities with SGD with the hinge loss or with the covariance loss on a two-layers fully connected network with a computational cost of $d^{O(1)}$;
    \item We empirically verify the effectiveness of our curriculum strategy for a set of fully connected architectures and parameters;
    \item We propose a class of functions - the \textit{Hamming mixtures} - that is provably not learnable by some curriculum strategies with finitely many learning steps. We conjecture that a \textit{continuous} curriculum strategy (see Def.~\ref{def:C-CL}) may allow to significantly improve the performance for learning such class of functions. 
\end{enumerate}

\subsection{Related Work}

\textbf{Learning parities on uniform inputs.}  
Learning $k$-parities over $d$ bits requires determining the set of relevant features among ${d \choose k}$ possible sets. The statistical complexity of this problem is thus $\theta(k \log(d))$. The computational complexity is harder to determine. $k$-parities can be solved in $d^{O(1)}$ time by specialized algorithms (e.g. Gaussian elimination) that have access to at least $d$ samples. In the statistical query (SQ) framework~\cite{kearns1998efficient} - i.e. when the learner has access only to noisy queries over the input distribution - $k$-parities cannot be learned in less then $\Omega(d^k)$ computations. \cite{AS20,shalev2017failures} showed that gradient-based methods suffer from the same SQ computational lower bound if the gradient precision is not good enough. On the other hand, \cite{AS20} showed that one can construct a very specific network architecture and initialization that can learn parities beyond this limit. This architecture is however far from the architectures used in practice. \cite{barak2022hidden} showed that SGD can learn sparse $k$-parities with SGD with batch size $d^{\theta(k)}$ on a small network.
Moreover, they empirically provide evidence of `hidden progress' during training, ruling out the hypothesis of SGD doing random search. 
\cite{andoni2014learning} showed that parities are learnable by a $d^{\theta(k)}$ network. The problem of learning \textit{noisy} parities (even with small noise) is conjectured to be intrinsically computationally hard, even beyond SQ models~\cite{alekhnovich2003more}. 

\textbf{Learning parities on non-uniform inputs.} Several works showed that when the input distribution is not the $\Unif \{ \pm 1 \}^d$, then neural networks trained by gradient-based methods can efficiently learn parities. \cite{malach2021quantifying} showed that biased parities are learnable by SGD on a differentiable model consisting of a linear predictor and fixed module implementing the parity. 
\cite{daniely2020learning} showed that sparse parities are learnable on a two layers network if the input coordinates outside the support of the parity are uniformly sampled and the coordinates inside the support are correlated. To the best of our knowledge, none of these works propose a curriculum learning model to learn parities under the uniform distribution. 


\textbf{Curriculum learning.} \textit{Curriculum Learning (CL)} in the context of machine learning has been extensively analysed from the empirical point of view~\cite{bengio2009curriculum,wang2021survey,soviany2022curriculum}. However, theoretical works on CL seem to be more scarce. In~\cite{saglietti2022analytical} the authors propose an analytical model for CL for functions depending on a sparse set of relevant features. In their model, easy samples have low variance on the irrelevant features, while hard samples have large variance on the irrelevant features. In contrast, our model does not require knowledge of the target task to select easy examples. In~\cite{weinshall2018curriculum,weinshall2020theory} the authors analyse curriculum learning strategies in convex models and show an improvement on the speed of convergence of SGD.
In contrast, our work covers an intrinsically non-convex problem.
Some works also analysed variants of CL: e.g. self-paced CL (SPCL), i.e. curriculum is determined by both prior knowledge and the training process~\cite{jiang2015self}, implicit curriculum, i.e. neural networks tend to consistently learn the samples in a certain order~\cite{toneva2018empirical}. To a different purpose,~\cite{abbe2021staircase,abbe2022merged} analyse staircase functions - sum of nested monomials of increasing degree - and show that the hierarchical structure of such tasks guides SGD to learn high degree monomials. Moreover,~\cite{refinetti2022neural,kalimeris2019sgd} show that SGD learns functions of increasing complexity during training. In a concurrent work~\cite{abbe2023generalization}, the authors propose a curriculum learning algorithm (named `Degree Curriculum') that consists of training on Boolean inputs of increasing Hamming weight, and they empirically show that it reduces the sample complexity of learning parities on small input dimension. However, the paper does not include theoretical analysis.

\section{Definitions and Main Results}
We define a curriculum strategy for learning a general Boolean target function. We will subsequently restrict our attention to the problem of learning parities or mixtures of parities. For brevity, we denote $[d] = \{ 1,...,d\}$. Assume that the network is presented with samples $(x, f(x))$, where $ x \in \{ \pm 1 \}^d$ is a Boolean vector and $f: \{ \pm 1 \}^d \to \bR$ is a target function that generates the labels. We consider a neural network $\NN(x;\theta)$, whose parameters are initialized at random from an initial distribution $P_0$, and trained by stochastic gradient descent (SGD) algorithm, defined by:
\begin{align} \label{eq:SGD}
    \theta^{t+1} = \theta^t - \gamma_t \frac{1}{B} \sum_{i=1}^B \nabla_{\theta^t} L(\theta^t, f,x^t_i),
\end{align}
for all $t \in \{0,...,T-1\}$, where $L$ is an almost surely differentiable loss-function, $\gamma_t$ is the learning rate, $B$ is the batch size and $T$ is the total number of training steps. For brevity, we write $L(\theta^t ,f,x):=L(\NN(.;\theta^t), f,x)$. We assume that for all $i \in [B]$, $x_i^t \overset{iid}{\sim} \cD^t$, where $\cD^t$ is a step-dependent input distribution supported on $\{ \pm  1 \}^d$. We define our curriculum learning strategy as follows. Recall that $z \sim \Rad(p)$ if $\pr(z =1) = 1-\pr(z=-1) = p$.
\begin{definition}[r-steps curriculum learning (r-CL)]  \label{def:r-CL}
    For a fixed $r \in \bN$, let $T_1,...T_r \in \bN$ and $p_1,...,p_r \in [0,1]$. Denote by $ \bar p: = (p_1,...,p_r)$ and $\bar T := (T_1,...,T_{r-1})$. We say that a neural network $\NN(x;\theta^t)$ is trained by SGD with a r-CL$(\bar T, \bar p)$ if $\theta^t$ follows the iterations in~\eqref{eq:SGD} with:
    \begin{align*}
            & \cD^t = \Rad(p_1) , \qquad 0 < t\leq T_1,\\
            &\cD^t = \Rad(p_2) , \qquad T_1 < t\leq T_2,\\
            & \cdots\\
            &\cD^t =\Rad(p_r) , \qquad  T_{r-1} < t\leq T.
    \end{align*}
    We say that $r$ is the number of \textit{curriculum steps}.
\end{definition}
\noindent 
We assume $r$ to be independent on $T$, in order to distinguish the $r$-CL from the \textit{continuous}-CL (see Def.~\ref{def:C-CL} below).
We hypothesize that $r$-CL may help to learn several Boolean functions, if one chooses appropriate $r$ and $\bar p$. However, in this paper we focus on the problem of learning unbiased $k$-parities. For such class, we obtained that choosing $r=2$, a wise $p_1 \in (0,1/2)$ and $p_2=1/2$ brings a remarkable gain in the computational complexity, compared to the standard setting with no curriculum. 
An interesting future direction would be studying the optimal $r$ and $\bar p$. Before stating our Theorem, let us clarify the generalization error that we are interested in. As mentioned before, we are interested in learning the target over the uniform input distribution.
\begin{definition}[Generalization error]
We say that SGD on a neural network $\NN(x ; \theta)$ learns a target function $f: \{\pm 1 \}^d \to \bR$ with $r$-CL$(\bar T, \bar p)$ up to error $\epsilon$, if it outputs a network $\NN(x;\theta^T)$ such that: 
\begin{align}
        \E_{x \sim \Rad(1/2)^{\otimes d}} \left[ L(\theta^T, f, x)\right] \leq \epsilon,
\end{align}
where $L$ is any loss function such that $ \E_{x \sim \Rad(1/2)^{\otimes d}} [L(f,f,x)] =0$.
\end{definition}
We state here our main theoretical result informally. We refer to Section~\ref{sec:parities_theoretical} for the formal statement with exact exponents and remarks. 

\begin{theorem}[Main positive result, informal]
    There exists a 2-CL strategy such that a 2-layer fully connected network of $d^{O(1)}$ size trained by SGD with batch size $d^{O(1)}$ can learn any $k$-parities (for $k$ even) up to error $\epsilon$ in at most $d^{O(1)}/\epsilon^2$ iterations.
\end{theorem}
Let us analyse the computational complexity of the above. At each step, the number of computations performed by a 2-layer fully connected network is given by:
    \begin{align}
        (dN+N) \cdot B,
    \end{align}
where $d$ is the input size, $N$ is the number of hidden neurons and $B$ is the batch size. Multiplying by the total number of steps and substituting the bounds from the Theorem we get that we can learn the $k$-parity problem with a 2-CL strategy in at most $d^{O(1)}$ total computations. Specifically, $O(1)$ denotes quantities that do not depend on $k$ or on $d$, and the statement holds also for large $k,d$. 
We prove the Theorem in two slightly different settings, see Section~\ref{sec:parities_theoretical}. 

One may ask whether the $r$-CL strategy is beneficial for learning general target tasks (i.e. beyond parities). While we do not have a complete picture to answer this question, we propose a class of functions for which some $r$-CL strategies are not beneficial. We call those functions the \textit{Hamming mixtures}, and we define them as follows.

\begin{figure*}[h]
    \centering 
    \includegraphics[width = 0.45 \textwidth]{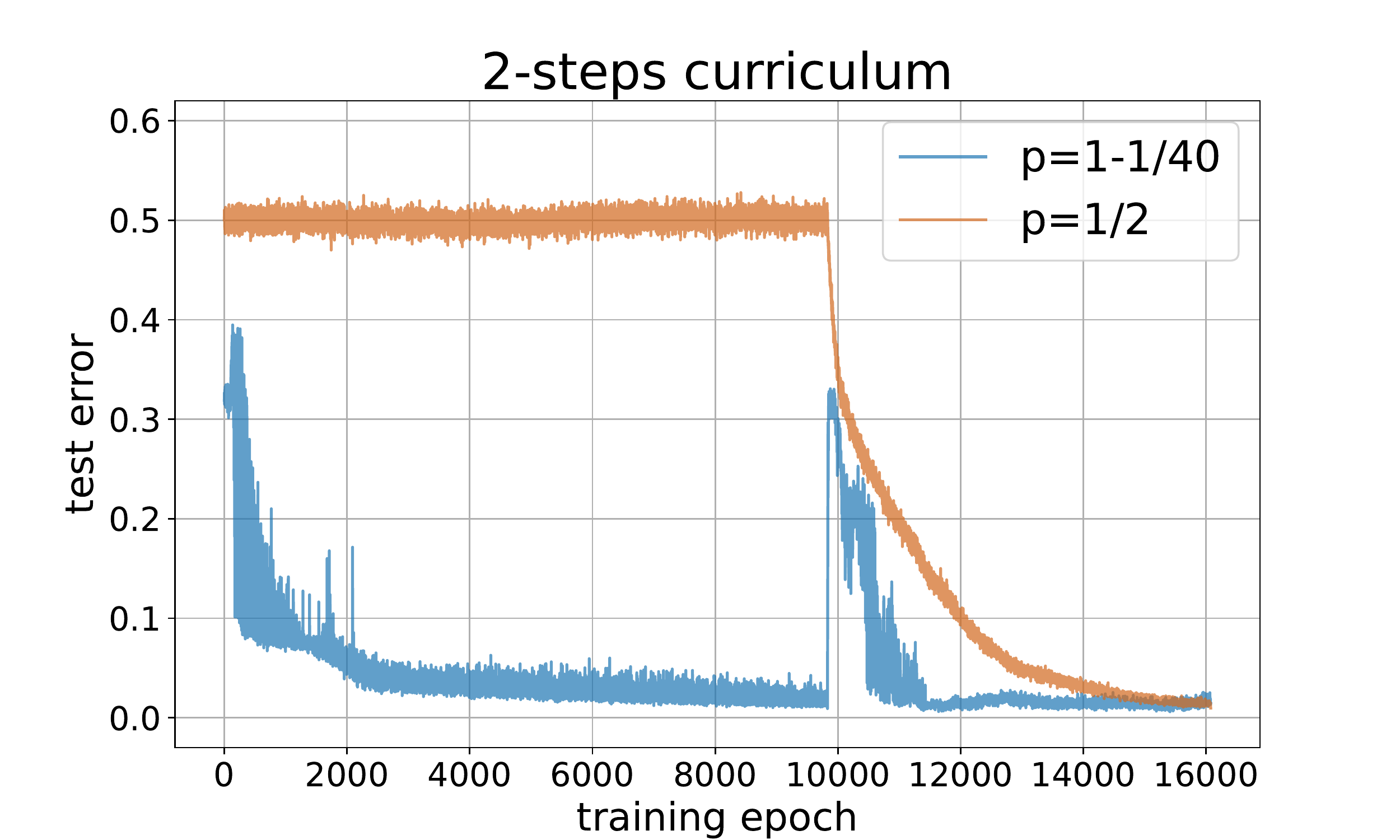}
    \includegraphics[width = 0.45 \textwidth]{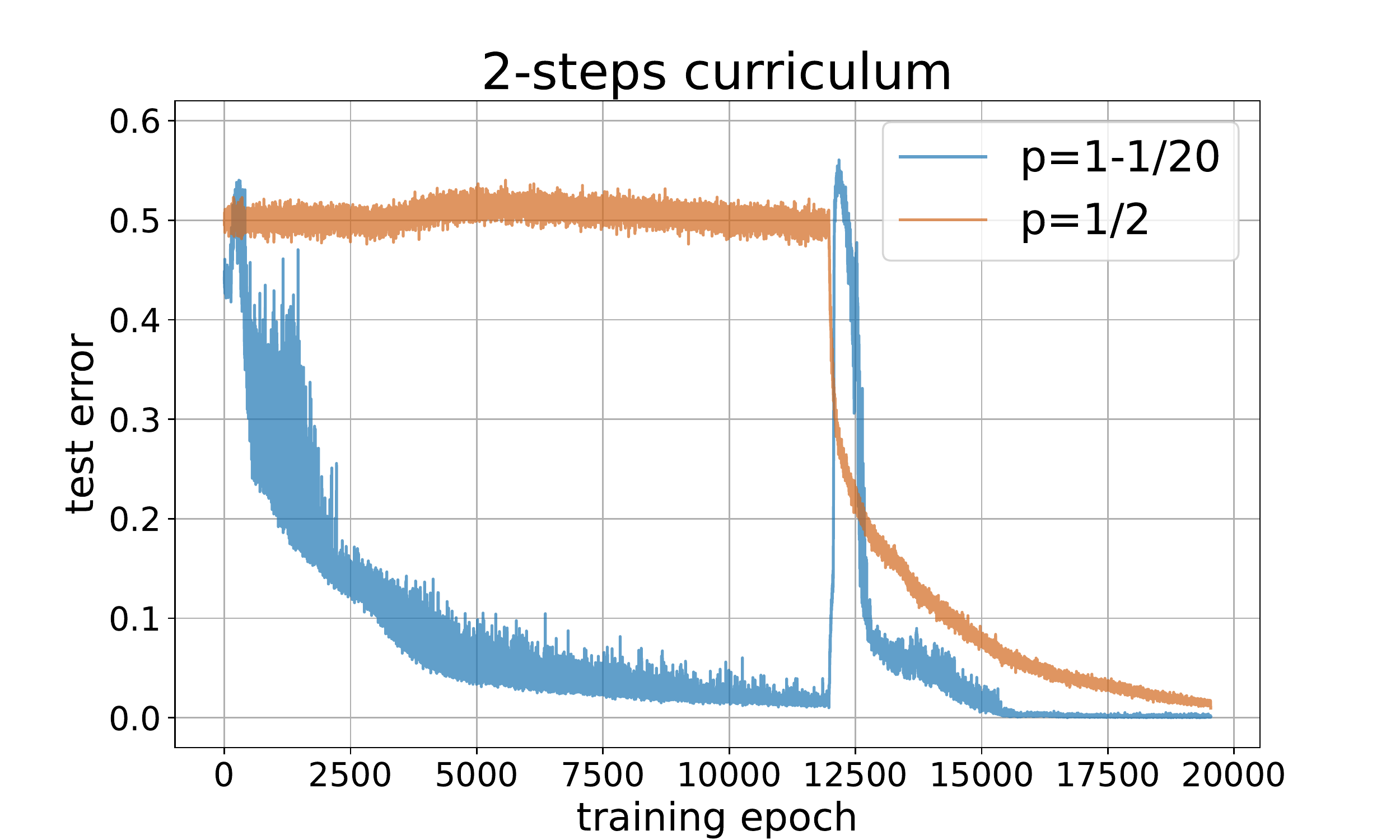}
    \includegraphics[width = 0.45 \textwidth]{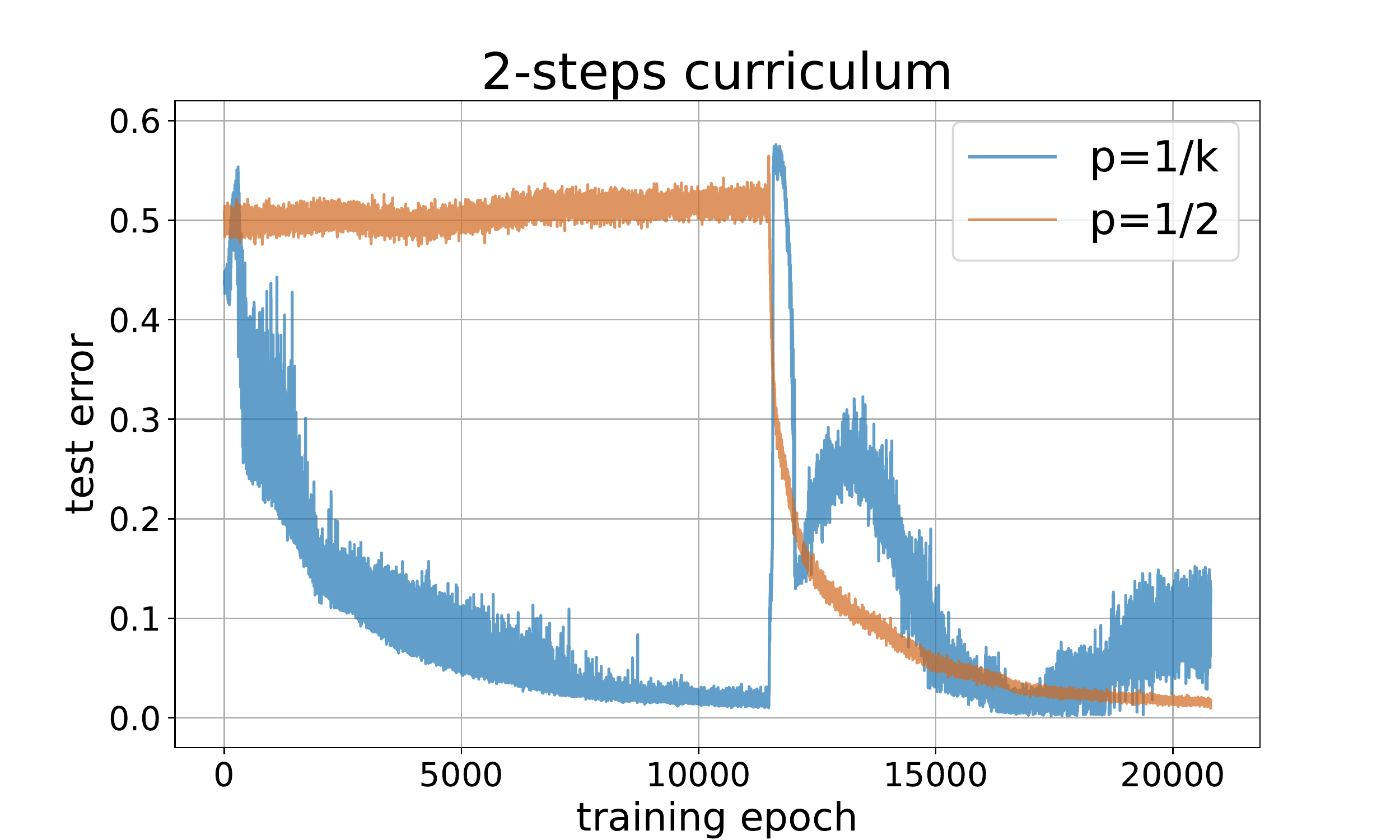}
    \includegraphics[width = 0.45 \textwidth]{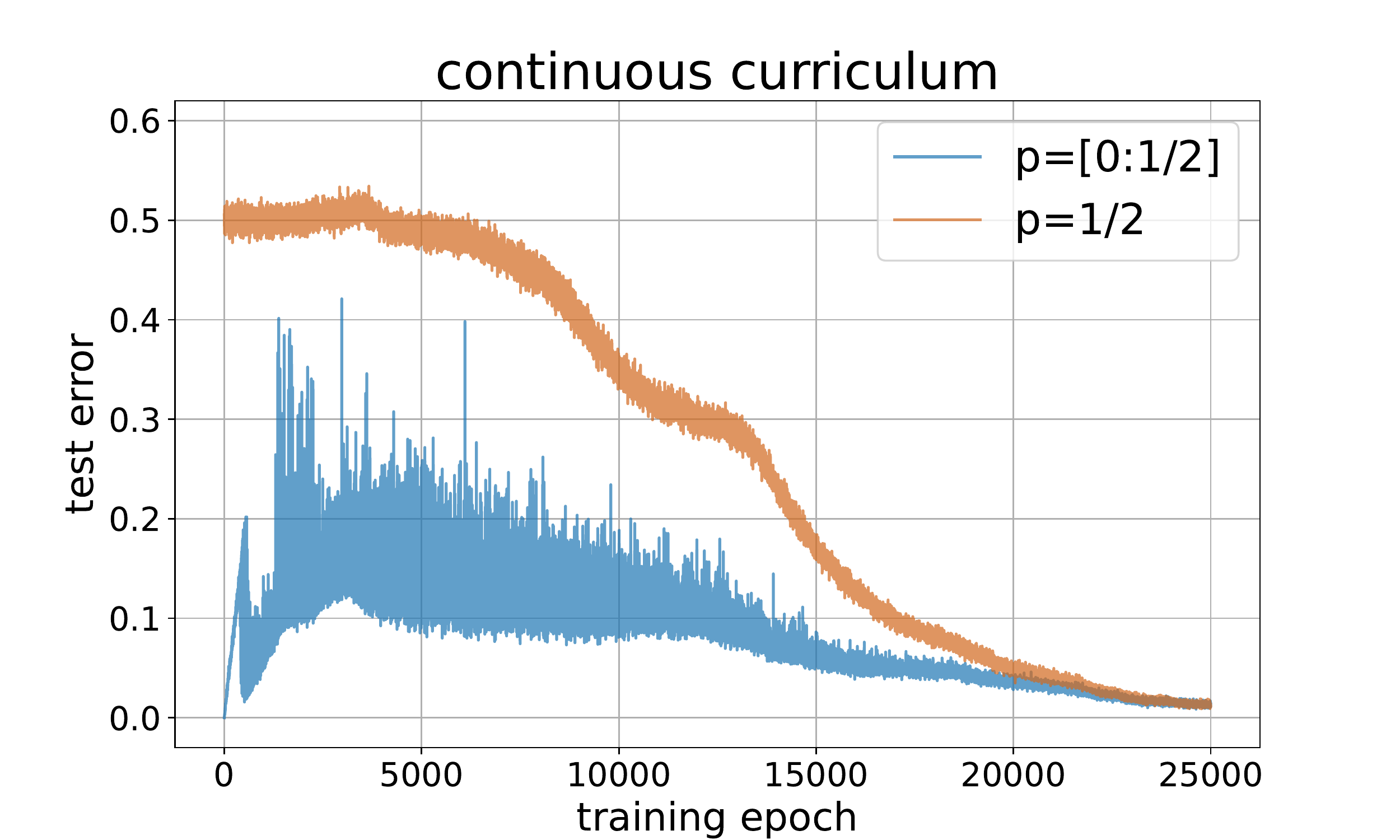}
    \includegraphics[width = 0.45 \textwidth]{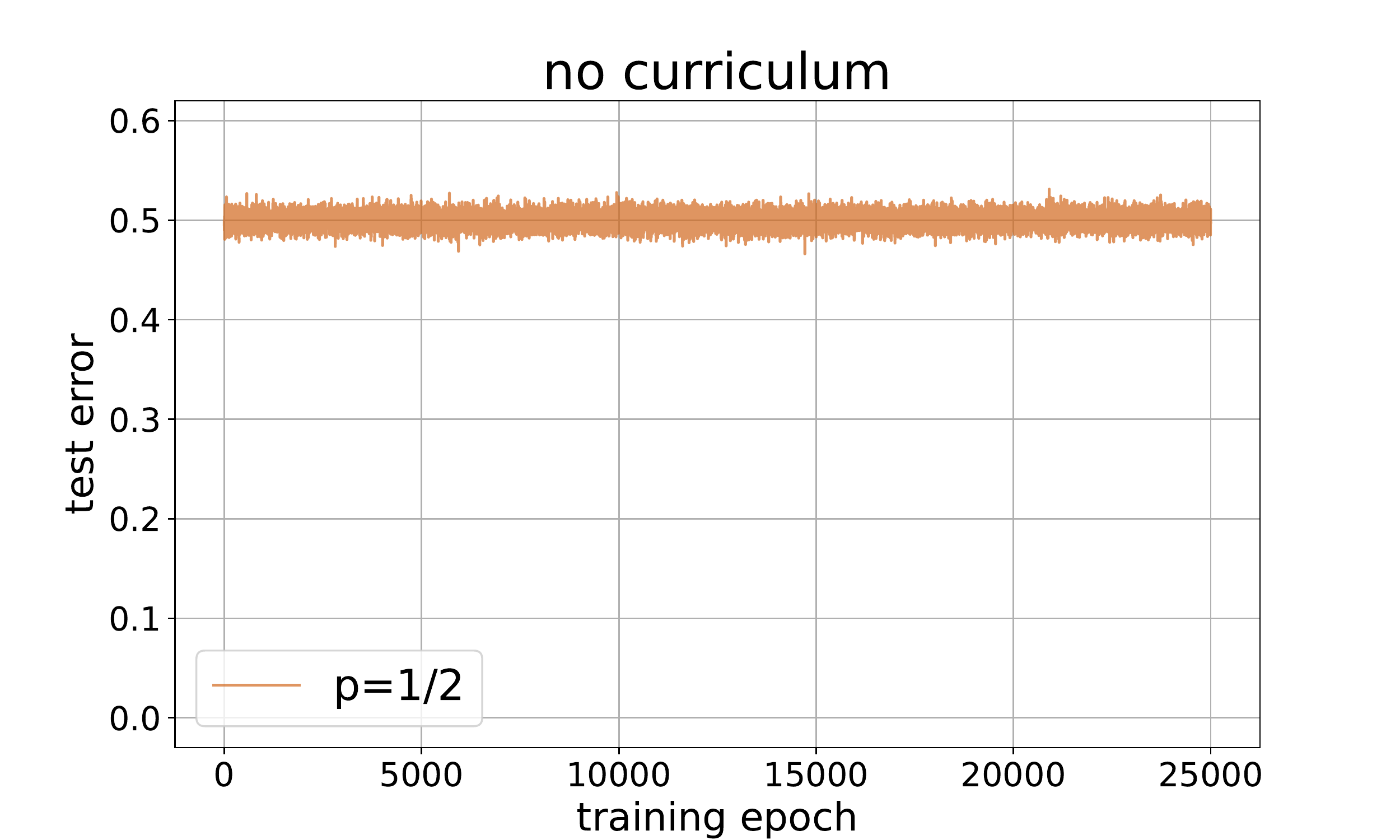}
    \caption{Learning $20$-parities with $2$-steps curriculum, with initial bias $p_1=39/40$ (top-left), $p_1=19/20$ (top-center), $p_1=1/20$ (top-right), with continuous curriculum (bottom-left) and with no curriculum (bottom-right). In all plots, we use a 2-layers ReLU MLP with batch size 1024, input dimension 100, and 100 hidden units. }
    \label{fig:curriculum_gen}
\end{figure*}
\begin{definition}[(S,T,$\epsilon)$-Hamming mixture] 
\label{def:Hamming_mixture}
    For $\epsilon \in [0,1]$, $S,T \in [d]$, we say that $G_{S,T,\epsilon}: \{\pm 1 \}^d \to \bR$ is a (S,T,$\epsilon)$-Hamming mixture if 
    \begin{align*}
        G_{S,T,\epsilon} (x):= \chi_S(x) 1(H(x) \leq \epsilon d) + \chi_T(x) 1(H(x) > \epsilon d),
    \end{align*}
    where $H(x) := \sum_{i=1}^d 1(x_i = 1) $ is the Hamming weight of $x$, $\chi_S(x) : = \prod_{i \in S} x_i$ and $\chi_T(x) := \prod_{i \in T} x_i$ are the parity functions over set $S$ and $T$ respectively.
\end{definition}
The intuition of why such functions are hard for some $r$-CL strategies is the following. Assume we train on samples $(x, G_{S,T.\epsilon}(x))$, with $S,T$ disjoint and $\epsilon \in (0,1/2)$. Assume that we use a $2$-CL strategy and we initially train on samples $x \sim \Rad(p)^{\otimes d} $ for some $p<\epsilon$. If the input dimension $d$ is large, then the Hamming weight of $x$ is with high probability concentrated around $ pd$ (e.g. by Hoeffding's inequality). Thus, in the first part of training
the network will see, with high probability, only samples of the type $(x, \chi_S(x))$, and it will not see the second addend of $G_{S,T, \epsilon}$. When we change our input distribution to $\Rad(1/2)^{\otimes d}$, the network will suddenly observe samples of the type $(x, \chi_T(x))$. Thus, the pre-training on $p$ will not help determining the support of the new parity $\chi_T$ (in some sense the network will ``forget'' the first part of training). This intuition holds for all $r$-CL such that $p_1,...,p_{r-1}<\epsilon$. We state our negative result for Hamming mixtures here informally, and refer to Section~\ref{sec:Hamming} for a formal statement and remarks.
\begin{theorem}[Main negative result, informal]
    For each $r$-CL strategy with $r$ bounded, there exists a Hamming mixture $G_{S,T,\epsilon}$ that is not learnable by any 
    fully connected neural network of $\poly(d)$ size and permutation-invariant initialization trained by the noisy gradient descent algorithm (see Def.~\ref{def:noisyGD}) with $\poly(d)$ gradient precision in $\poly(d)$ steps.
\end{theorem}

Inspired by the hardness of Hamming mixtures, we define another curriculum learning strategy, where, instead of having finitely many discrete curriculum steps, we gradually move the bias of the input distribution during training from a starting point $p_0$ to a final point $p_T$. We call this strategy a \textit{continuous}-CL strategy.

\begin{definition}[Continuous curriculum learning (C-CL)] \label{def:C-CL}
    Let $p_0, p_T \in [0,1]$. We say that a neural network $\NN(x;\theta^t)$ is trained by SGD with a C-CL$(p_0,p_T,T) $ if $\theta^t$ follows the iterations in~\eqref{eq:SGD} with:
    \begin{align}
        & \cD^t = \Rad \left(p_0+t \cdot \frac{p_T-p_0}{T} \right) \qquad t \in [T].
    \end{align}
\end{definition}
\noindent
We conjecture that a well chosen C-CL might be beneficial for learning any Hamming mixture. A positive result for C-CL and comparison between $r$-CL and C-CL are left for future work.


\section{Learning Parities}

\begin{figure*}[h]
\centering
    \includegraphics[width = 0.45 \textwidth]{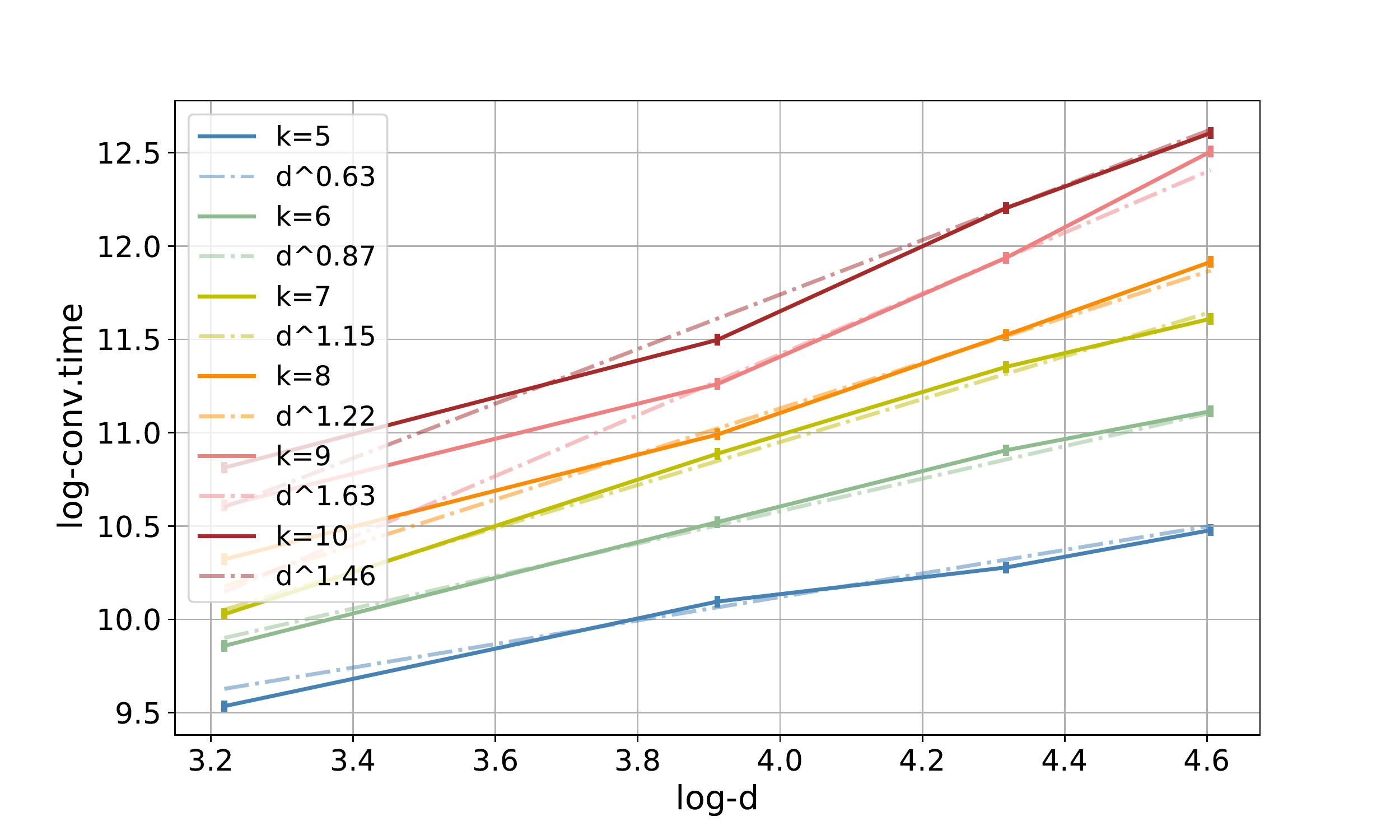}
    \includegraphics[width = 0.45 \textwidth]{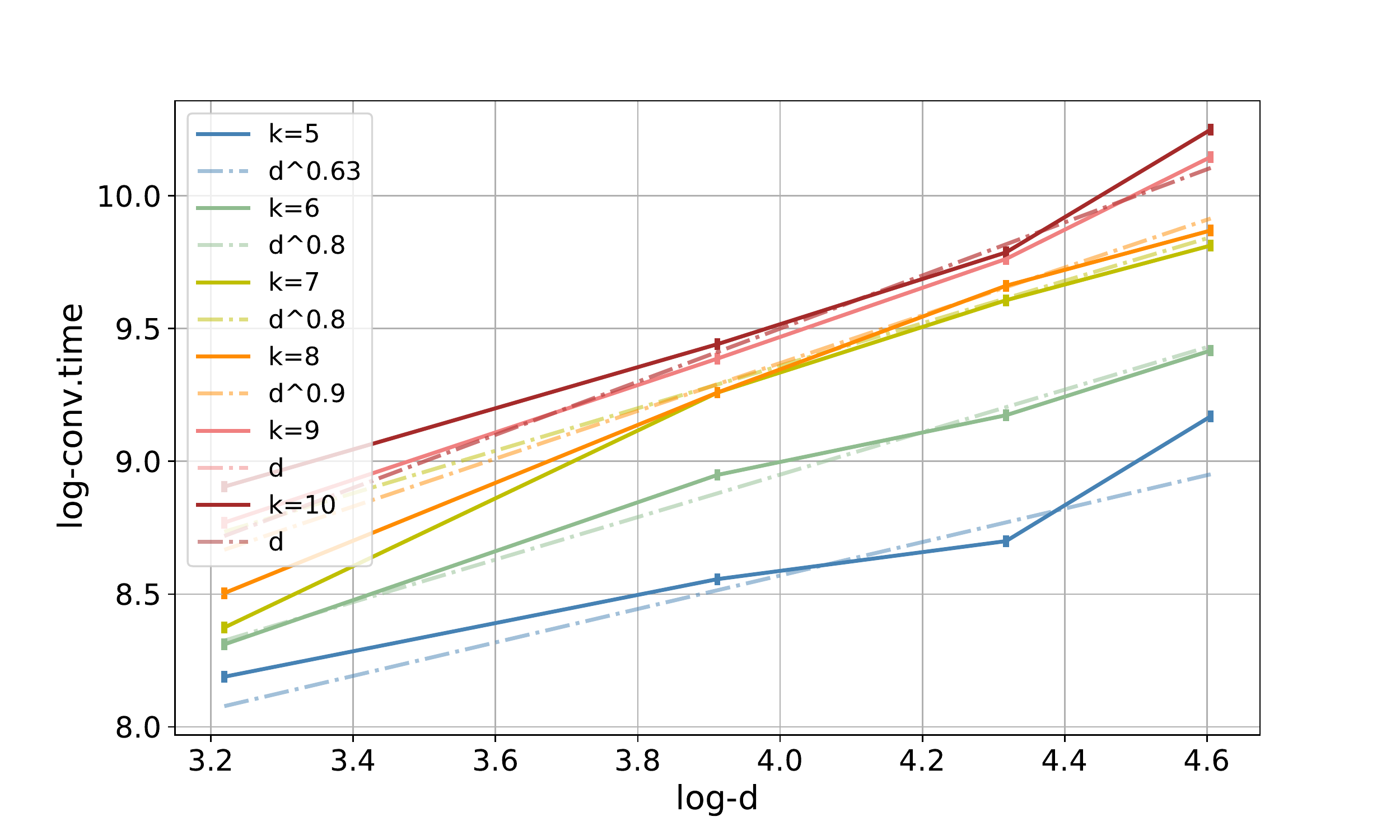}
    \caption{Convergence time for different values of $d$, $k$. Left: we take $p_1 = 1/16$ and a 2-layers $\ReLU$ architecture with with $h=2^k$ hidden units. Right: we take $p_1 = 1-\frac{1}{2k}$ and a 2-layers $\ReLU$ architecture with $h=d$ hidden units.}
    \label{fig:power_law}
\end{figure*}
\subsection{Theoretical Results}
\label{sec:parities_theoretical}
Our goal is to show that the curriculum strategy that we propose allows to learn $k$-parities with a computational complexity of $d^{O(1)}$. We prove two different results. In the first one, we consider SGD on the hinge loss and prove that a network with $\theta(d^2) $ hidden units can learn the $k$-parity problem in $d^{O(1)}$ computations, if trained with a well chosen $2$-CL strategy. Let us state our first Theorem.
\begin{theorem}[Hinge Loss] \label{thm:positive_result}
    Let $k,d$ be both even integers, such that $k \leq d/2$. Let $\NN(x;\theta) = \sum_{i=1}^N a_i \sigma(w_ix + b_i) $ be a 2-layers fully connected network with activation $\sigma(y) := \Ramp(y) $ (as defined in~\eqref{eq:ramp}) and $N = \tilde  \theta(d^2 \log(1/\delta))$\footnote{$\tilde \theta(d^{c}) = \theta(d^{c} \cdot  \poly(\log(d)))  $, for all $c \in \bR$.}. Consider training $\NN(x;\theta) $ with SGD on the hinge loss with batch size $B = \tilde \theta (d^{10}/\epsilon^2 \log(1/\delta) )$. Then, there exists an initialization, a learning rate schedule, and a 2-CL strategy such that after $T = \tilde \theta(d^6 /\epsilon^2)$ iterations, with probability $1-3\delta$, SGD outputs a network with generalization error at most $\epsilon$.
\end{theorem}
For our second Theorem, we consider another loss function, that is convenient for the analysis, namely the \textit{covariance loss}, for which we give a definition here.
\begin{definition} [Covariance loss]\label{def:covariance_loss}
    Let $f: \cX \to \bR$ be a target function and let $\hat f: \cX \to \bR $ be an estimator. Let
    \begin{align*}
        \Cov &(f,\hat f,x,P_\cX) :=   \\
            &: = \Big(f(x) - \E_{x' \sim P_\cX}[f(x') ] \Big)  \cdot \Big(\hat f(x) - \E_{x' \sim P_\cX}[\hat f(x') ] \Big),
    \end{align*}
    where $P_\cX$ is an input distribution supported in $\cX$.
    We define the covariance loss as 
       \begin{align*}
            L_{\Cov}(f,\hat f,x,P_\cX)& := \max\{ 0,1- \Cov(f, \hat f,x,P_\cX) \}.
       \end{align*}

\end{definition}

\begin{remark}
    We will consider optimization over the covariance loss through SGD with large batch size ($B = \tilde \theta(d^2 k^3)$). At each step, we use the batch to estimate first the inner expectations (i.e. $E_{x}[f(x)]$ and $E_{x}[\NN(x;\theta^t)]$) and then the gradients. The expectation of the labels (i.e. $E_{x}[f(x)]$) does not need to be estimated at each training step and could be estimated once per curriculum step. One could also use part of the batch at each step to estimate the inner expectations and part of the batch to estimate the gradients.
\end{remark}
We show that SGD on the covariance loss can learn the $k$-parity problem in $d^{O(1)}$ computations using a network with only $O(k)$ hidden units. The reduction of the size of the network, compared to the hinge loss case, allows to get a tighter bound on the computational cost, see Remark~\ref{rem:2_theorems}.
\begin{theorem}[Covariance Loss] \label{thm:positive_cov}
       Let $k,d$ be integers, such that $k \leq d$ and $k$ even. Let $\NN(x;\theta) = \sum_{i=1}^N a_i \sigma(w_ix + b_i) $ be a 2-layers fully connected network with activation $\sigma(y) := \ReLU(y) $ and $N = \tilde \theta(k)$. Consider training $\NN(x;\theta) $ with SGD on the covariance loss with batch size $B =\tilde \theta(d^2 k^3/\epsilon^2 \log(1/\delta))$. Then, there exists an initialization, a learning rate schedule, and a 2-CL strategy such that after $T = \tilde \theta(k^4/\epsilon^2 )$ iterations, with probability $1-3\delta$, SGD outputs a network with generalization error at most $\epsilon$.
\end{theorem}


The proofs of Theorem~\ref{thm:positive_result} and Theorem~\ref{thm:positive_cov} follow a similar outline. Firstly, we prove that training the first layer of the network on one batch of size $d^{O(1)}$ sampled from a biased input distribution (with appropriate bias), allows to recover the support of the parity. We then show that training the second layer on the uniform distribution allows to achieve the desired generalization error under the uniform distribution. 
We refer to Appendices~\ref{sec:proof_positive_hinge} and~\ref{sec:proof_positive_cov} for restatements of the Theorems and their full proofs.

\begin{remark} \label{rem:2_theorems}
Let us look at the computational complexity given by the two Theorems. Theorem~\ref{thm:positive_result} tells that we can learn $k$-parities in $ d N B + (T-1)N = \tilde \theta (d^{19})$ computations. We remark that our result holds also for large $k$ (we however need to assume $k,d$ even and $k \leq d/2$, for technical reasons). On the other hand, Theorem~\ref{thm:positive_cov} tells that we can learn $k$-parities in $ \tilde \theta(d^3 k^8)$, which is much lower than the bound given by Theorem~\ref{thm:positive_result}. Furthermore, the proof holds for all $k\leq d$. The price for getting this tighter bound is the use of a loss that (to the best of our knowledge) is not common in the machine learning literature, and that is particularly convenient for our analysis.
\end{remark}
\begin{remark}
We remark that our proofs extend to the gradient descent model with bounded gradient precision, used in~\cite{AS20}, with gradient precision bounded by $d^{O(1)}$. Thus, for large $k,d$, our result provides a separation to their $d^{\Omega(k)}$ computational lower bound for learning $k$-parities under the uniform distribution with no curriculum. 
\end{remark}

\begin{remark}
Let us comment on the $p_1$ (i.e. the bias of the initial distribution) that we used. In both Theorems we take $p_1$ close to $1$. In Theorem~\ref{thm:positive_result} we take $p_1 \approx 1-\theta(1/d)$, and the proof is constructed specifically for this value of $p_1$. In Theorem~\ref{thm:positive_cov}, the proof holds for any $p_1 \in (1/2,1)$ and the asymptotic complexity in $d$ does not depend on the specific choice of $p_1$. However, to get $\poly(k)$ complexity we need to take $p_1 = 1- \theta(1/k)$, while we get $\exp(k)$ complexity for all $p_1 = \theta_{d,k}(1)$.
\end{remark}
Our theoretical analysis captures a fairly restricted setting: in our proofs we use initializations and learning schedules that are convenient for the analysis. We conduct experiments to verify the usefulness of our CL strategy in more standard settings of fully connected architectures.


\subsection{Empirical Results}
In all our experiments we use fully connected $\ReLU$ networks and we train them by SGD on the square loss~\footnote{Code: \url{https://github.com/ecornacchia/Curriculum-Learning-for-Parities}}.

In Figure~\ref{fig:curriculum_gen}, we compare different curriculum strategies for learning $20$-parities over $100$ bits, with a fixed architecture, i.e. a $2$-layer $\ReLU$ network with $100$ hidden units. We run a $2$-steps curriculum strategy for $3$ values of $p_1$, namely $p_1 = 39/40,19/20,1/20$. In all the $2$-CL experiments we train on the biased distribution until convergence, and then we move to the uniform distribution. We observe that training with an initial bias of $p_1 =39/40$ allows to learn the $20$-parity in $16,000$ epochs. One can see that during the first part of training (on the biased distribution), the test error under the uniform distribution stays at $1/2$ (orange line), and then drops quickly to zero when we start training on the uniform distribution. This trend of hidden progress followed by a sharp drop has been already observed in the context of learning parities with SGD in the standard setting with no-curriculum~\cite{barak2022hidden}. Here, the length of the `hidden progress' phase is controlled by the length of the first phase of training. Interestingly, when training with continuous curriculum, we do not have such hidden progress and the test error under the uniform distribution decreases slowly to zero. With no curriculum, the network does not achieve non-trivial correlation with the target in $25,000$ epochs. 

In Figure~\ref{fig:power_law} we study the convergence time of a $2$-CL strategy on a $2$-layers $\ReLU$ network for different values of the input dimension ($d$) and size of the parity ($k$). We take two slightly different settings. In the plot on the left, we take a fixed initial bias $p_1 = 1/16$ and $h=2^k$ hidden units. On the right we take $p_1 =1-\frac{1}{2k}$ initial bias and an architecture with $h=d$ hidden units. The convergence time is computed as $T_1+T_2$, where $T_1$ and $T_2$ are the number of steps needed to achieve training error below $0.01$ in the first and second part of training, respectively. We compute the convergence time for $k=5,6,7,8,9,10$ and $d=25,50,75,100$, and for each $k$ we plot the convergence time with respect to $d$ in log-log scale. Each point is obtained by averaging over $10$ runs. We observe that for each $k$, the convergence time scales (roughly) polynomially as $d^{c_k}$, with $c_k$ varying mildly with $k$.
\begin{figure}[t]
\centering
\includegraphics[width = 0.45\textwidth]{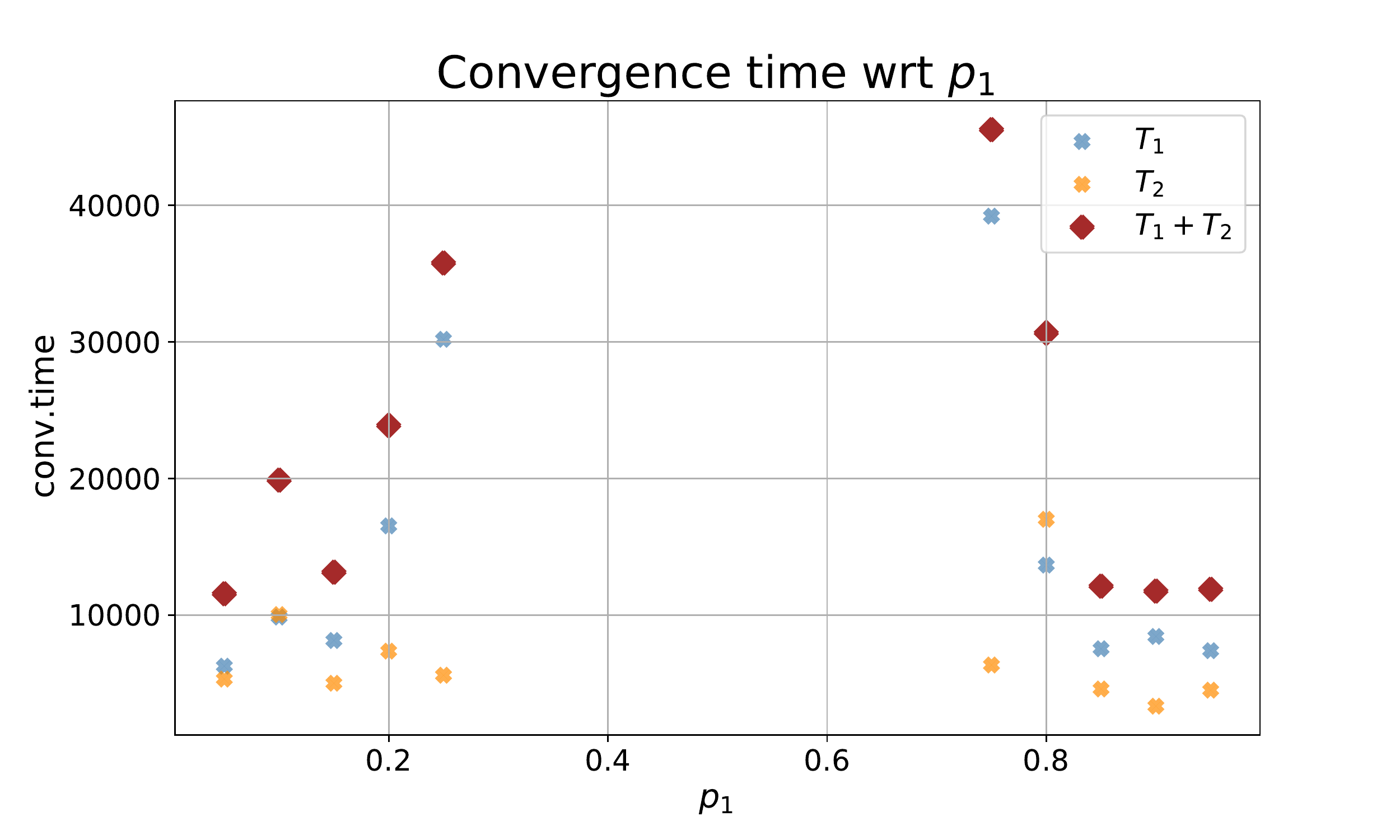}
\caption{Convergence time with respect to the initial bias $p_1$. We compute the convergence time for learning a $10$-parity over $100$ bits with a 2-layer $\ReLU$ network. We omitted all points with convergence time above $100,000$.}
\label{fig:best_p}
\end{figure}

In Figure~\ref{fig:best_p}, we study the convergence time of a $2$-CL strategy for different values of the initial bias $p_1$. We consider the problem of learning a $10$-parity over $100$ bits with a $2$-layers $\ReLU $ network with $h=100$ hidden units. As before, we computed the convergence time as $T_1+T_2$, where $T_1$ and $T_2$ are the number of steps needed to achieve training error below $0.01$ in the first and second part of training, respectively. We ran experiments for $p_1 = 0.001, 0.05,0.1,0.15,...,0.95,0.999$. We omitted from the plot any point for which the convergence time exceeded $100,000$ iterations: these correspond to $p_1$ near $1/2$ and $p_1 =0.001,0.999$. Each point is obtained by averaging over $10$ runs. We observe that the convergence time is smaller for $p_1$ close to $0$ or to $1$. Moreover, $T_2$ has modest variations across different $p_1$'s.

\section{Learning Hamming Mixtures} 
\label{sec:Hamming}
In this section we consider the class of functions defined in Def.~\ref{def:Hamming_mixture} and named Hamming mixtures. 
We consider a specific descent algorithm, namely the noisy GD algorithm with batches (used also in~\cite{AS20,abbe2021power}). We give a formal definition here of noisy GD with curriculum. 
\begin{definition}[Noisy GD with CL] \label{def:noisyGD}
Consider a neural network $\NN(.;\theta)$, with initialization of the weights $\theta^{0}$. Given an almost surely differentiable loss function, the updates of the noisy GD algorithm with learning rate $\gamma_t$ and gradient range $A$ are defined by
\begin{align} \label{eq:noisy_GD}
    \theta^{t+1} = \theta^{t}- \gamma_t \left( \E_{x^t}  [ \nabla_{\theta^t} L(\theta^t,f,x^t) ]_A + Z^{t}\right), 
\end{align}
where for all $t \in \{0,...,T-1\}$, $Z^{t}$ are i.i.d. $\cN(0, \tau^2)$, for some $\tau$, and they are independent from other variables,
$x^t \sim \cD^t$, for some time-dependent input distribution $\cD^t$, $f$ is the target function, from which the labels are generated, and by $[.]_A$ we mean that whenever the argument is exceeding $A$ (resp. $-A$) it is rounded to $A$ (resp. $-A$). We call $A/\tau$ the \textit{gradient precision}. In the noisy-GD algorithm with $r$-CL, we choose $\cD^t$ according to Def.~\ref{def:r-CL}.
\end{definition}

\noindent
Let us state our hardness result for learning Hamming mixtures with $r$-CL strategies with $r$ bounded.
      
\begin{theorem} \label{thm:negative_result}
    Assume the network observes samples generated by $G_{S,T, \epsilon}(x)$ (see Def.~\ref{def:Hamming_mixture}), where $|S|=k_S$, $|T|=k_T$ such that $k_S,k_T = o(\sqrt{d})$, and $|S \cap T | =0$. Then, for any $r$-CL$(\bar T,\bar p)$ with $r$ bounded and $p_r=1/2$, there exists an $\epsilon$ such that the noisy GD algorithm with $r$-CL$( \bar T,\bar p)$ (as in~\eqref{eq:noisy_GD}) on a fully connected neural network with $|\theta|$ weights and permutation-invariant initialization, after $T$ training steps, outputs a network $\NN(x, \theta^T)$ such that
    \begin{align*}
        \Big|& \E_{x\sim \Rad(1/2)^{\otimes d}} \left[ G_{S,T,\epsilon}(x) \cdot \NN(x; \theta^T) \right] \Big| \\
        &\leq \frac{AT \sqrt{|\theta|}}{\tau } \left(\frac{1}{d^{k_T/2}}+ e^{-d \delta^2} \right)+ \frac{2 k_S k_T}{d} + O(d^{-2}),
    \end{align*}      
    where $A,\tau$ are the gradient range and the noise level in the noisy-GD algorithm and $\delta$ is a constant.
\end{theorem}
The proof uses an SQ-like lower bound argument for noisy GD, in a similar flavour of~\cite{AbbeINAL,abbe2022non-universality}. We refer to Appendix~\ref{sec:proof_negative} for the full proof.

\begin{remark}
In Theorem~\ref{thm:negative_result}, the neural network can have any fully connected architecture and any activation such that the gradients are well defined almost everywhere. The initialization can be from any distribution that is invariant to permutations of the input neurons.
\end{remark}
For the purposes of $\E \left[ G_{S,T,\epsilon}(x) \cdot \NN(x; \theta^T) \right]  $, it is assumed that the neural network outputs a guess in $\{ \pm 1 \}$. This can be done with any form of thresholding, e.g. taking the sign of the value of the output neuron.
\begin{remark}
One can remove the $\frac{2 k_S k_V}{d}$ term in the right hand side by further assuming e.g. that set $S$ is supported on the first $d/2$ coordinates and set $V$ on the last $d/2$ coordinates. This also allows to weaken the assumption on the cardinality of $S$ and $V$. We formalize this in the following Corollary.
\end{remark}

\begin{corollary} \label{cor:negative_nointersection}
    Assume the network observes samples generated by $G_{S,V, \epsilon}(x)$, where $S \subseteq \{1,...,d/2\}$, and $V \subseteq \{d/2+1,...,d\}$ (where we assumed $d$ to be even for simplicity). 
    Denote $k_V  = |V|$.
    Then, for any $r$-CL$(\bar T,\bar p)$ with $r$ bounded and $p_r=1/2$, there exists an $\epsilon$ such that the noisy GD algorithm with $r$-CL$( \bar T,\bar p)$ (as in~\eqref{eq:noisy_GD}) on a fully connected neural network with $|\theta|$ weights and permutation-invariant initialization, after $T$ training steps, outputs a network $\NN(x, \theta^{(T)})$ such that
    \begin{align*}
        \Big| \E_{x\sim \Rad(1/2)^{\otimes d}} &\left[ G_{S,V,\epsilon}(x) \cdot \NN(x; \theta^{(T)}) \right] \Big| \\
        &  \leq \frac{2AT \sqrt{|\theta|}}{\tau } \left({d/2 \choose k_V}^{-1/2}+ e^{-d \delta^2} \right),
    \end{align*} 
    for some $\delta >0$.
\end{corollary}
The proof of Corollary~\ref{cor:negative_nointersection} is deferred to Appendix~\ref{sec:proof_cor_negative}.


Theorem~\ref{thm:negative_result} states failure at the weakest form of learning, i.e. achieving correlation better than guessing in the asymptotic of large $d$. More specifically, it tells that if the network size, the number of training steps and the gradient precision (i.e. $A/\tau$) are such that $\frac{AT \sqrt{|\theta|}}{\tau} = o(d^{-k_T/2})$, then the network achieves correlation with the target under the uniform distribution of $o_d(1)$.
 Corollary~\ref{cor:weak_learning} follows immediately from the Theorem.
\begin{corollary} \label{cor:weak_learning}
    Under the assumptions of Theorem~\ref{thm:negative_result}, if $k_T = \omega_d(1)$ (i.e. $k_T$ grows with $d$), $|\theta|, A/\tau, T$ are  all polynomially bounded in $d$, then 
    \begin{align}
        \Big|& \E_{x\sim \Rad(1/2)^{\otimes d}} \left[ G_{S,T,\epsilon}(x) \cdot \NN(x; \theta^T) \right] \Big|  = o_d(1),
    \end{align}
    i.e. in $\poly(d)$ computations the network will fail at weak-learning $G_{S,T,\epsilon}$.
\end{corollary}
We conjecture that if we take instead a C-CL strategy with an unbounded number of curriculum steps, we can learn efficiently (i.e. in $\poly(d)$ time) any $G_{S,T,\epsilon}$ (even with $k_T=\omega_d(1)$ and for any $\epsilon$). 
Furthermore, we believe this conjecture to hold for any bounded mixture, i.e. any function of the type:
\begin{align}
    \sum_{m=1}^M \chi_{S_m} (x) \mathds{1}(\epsilon_{m-1} d \leq H(x) < \epsilon_m d),
\end{align}
with $S_1,...,S_M$ being distinct sets of coordinates, $0=\epsilon_0<\epsilon_1...<\epsilon_M \leq 1$, and $M $ bounded.
\section{Conclusion and Future Work}
In this work, we mainly focused on learning parities and Hamming mixtures with $r$-CL strategies with bounded $r$. Some natural questions arise, for instance: does the depth of the network help? What is the optimal number of curriculum steps for learning parities? 
We leave to future work the analysis of C-CL with unboundedly many curriculum steps and the comparison between $r$-CL and C-CL. In the previous Section, we also raised a conjecture concerning the specific case of Hamming mixtures. 

Furthermore, we believe that our results can be extended to more general families of functions. First, consider the set of $k$-Juntas, i.e., the set of functions that depend on $k$ out of $d$ coordinates. This set of functions contains the set of $k$-parities so it is at least as hard to learn. Moreover, as in the case of parities, Juntas are correlated with each of their inputs for generic $p$, see e.g.~\cite{mossel2004learning}.
So it is natural to expect that curriculum learning can learn such functions in time $d^{O(1)} 2^{O(k)}$ (the second term is needed since there is a doubly exponential number of Juntas on $k$ bits). In this work we propose to learn parities using a mixture of product distributions, but there are other ways to correlate samples that may be of interest. For example, some works in PAC learning showed that, even for the uniform measure, samples that are generated by a random walk often lead to better learning algorithms~\cite{bshouty2005learning,arpe2008agnostically}. 
Do such random walk based algorithms provide better convergence for gradient based methods? 

We further believe that a similar idea to the one presented in this paper can be applied to product distributions with orthogonal basis (such as Hermite monomials for the i.i.d. standard Gaussian distribution or spherical harmonics for the uniform distribution over a sphere). These basis elements are no longer orthogonal under biased distributions, and we anticipate that the footprints of our proof would extend to these scenarios. However, in real-world datasets, input coordinates are often not i.i.d., and each coordinate may depend on multiple other coordinates. Nevertheless, we are hopeful that in certain real-world datasets it may be possible to identify easy and hard samples by means of the variance of the input coordinates (i.e. 
$\frac{1}{d-1}\sum_{i=1}^d (x_i-\bar x)^2$ for $x \in \bR^d$). For instance, consider a task where a learner is required to identify a small object in an image (e.g. a `stop’ signal or a traffic light). In each image, the learner has to identify the relevant subset of coordinates and, intuitively, this is easier in images where the background is plain (samples with low variance) than in images where the background is noisy (samples with large variance). 

To conclude, we remark that an important limitation of the curriculum strategy presented in this paper is that it requires an oracle that provides labeled samples from arbitrary product measures. However, in applications one usually has a fixed dataset and would like to select samples in a suitable order, to facilitate learning. It would be an interesting future direction to consider settings where curriculum and non-curriculum have a common sampling distribution.

\section*{Acknowledgement}
This work was supported in part by the Simons-NSF Collaboration on the Theoretical Foundations of Deep Learning (deepfoundations.ai). 
It started while E.C. was visiting the MIT Institute for Data, Systems, and Society (IDSS) under the support of the collaboration grant. E.M is also partially supported by the Vannevar Bush Faculty Fellowship award ONR-N00014-20-1-2826 and by a Simons Investigator Award in Mathematics (622132).

\bibliography{references}
\bibliographystyle{alpha}

\newpage
\appendix
\onecolumn

\section{Proof of Theorem~\ref{thm:positive_result}}
\label{sec:proof_positive_hinge}
\begin{theorem}[Theorem~\ref{thm:positive_result}, restatement]
        Let $k,d$ be both even integers, such that $k \leq d/2$. Let $\NN(x;\theta) = \sum_{i=1}^N a_i \sigma(w_ix + b_i) $ be a 2-layers fully connected network with activation $\sigma(y) := \Ramp(y) $ (as defined in~\eqref{eq:ramp}) and $N \geq (d+1)(d-k+1) \log((d+1)(d-k+1)/\delta)$. Consider training $\NN(x;\theta) $ with SGD on the hinge loss with batch size $B \geq (8\zeta^2 N^2)^{-1}\log(\frac{Nd+N}{\delta})$, with $\zeta \leq \frac{\epsilon \mu^k }{24 (d+1)^2 (d-k+1)^2 N }$ and $\mu = \sqrt{1-\frac{1}{2(d-k)}}$. Then, there exists an initialization and a learning rate schedule, and a 2-CL strategy such that after $T \geq \frac{64}{\epsilon^2} (d-k+1)^3 (d+1) N$ iterations, with probability $1-3\delta$ SGD outputs a network with generalization error at most $\epsilon$.
\end{theorem}


\subsection{Proof Setup}
We consider a 2-layers neural network, defined as:
\begin{align}
    \NN(x; \theta ) = \sum_{i=1}^N a_i \sigma (w_i x +b_i ),
\end{align}
where $N$ is the number of hidden units, $\theta = (a,b,w)$ and $\sigma := \Ramp $ denotes the activation defined as:
\begin{align} \label{eq:ramp}
\Ramp(x) =  \begin{cases}
      0 \quad x \leq 0,\\
      x \quad 0<x \leq 1,\\
      1 \quad x >1
    \end{cases}\,.
\end{align}
Without loss of generality, we assume that the labels are generated by $\chi_{[k]}(x):= \prod_{i=1}^k x_i$. Indeed, SGD on fully connected networks with permutation-invariant initialization is invariant to permutation of the input neurons, thus our result will hold for all $\chi_S(x)$ such that $|S| =k$.
Our proof scheme is the following:
\begin{enumerate}
    \item We train only the first layer of the network for one step on data $(x_i, \chi_{[k]}(x_i))_{i \in [B]}$ with $x_i \sim \Rad(p)^{\otimes d}$ for $i \in [B]$, with $p =\frac{1}{2} \sqrt{1- \frac{1}{2(d-k)} } +\frac{1}{2}$;
    \item  We show that after one step of training on such biased distribution, the target parity belongs to the linear span of the hidden units of the network;
    \item We subsequently train only the second layer of the network on $(x_i, \chi_{[k]}(x_i))_{i \in [B]}$ with $x_i \sim \Rad(1/2)^{\otimes d}$ for $i \in [B]$, until convergence;
    \item We use established results on convergence of SGD on convex losses to conclude.
\end{enumerate}

We train our network with SGD on the hinge loss. Specifically, we apply the following updates, for all $t \in \{0,1,...,T-1\}$:
\begin{align}
        w_{i,j}^{t+1} &= w_{i,j}^{t} - \gamma_{t} \frac{1}{B} \sum_{s=1}^B \nabla_{w_{i,j}^{t}} L( \theta^{t}, \chi_{[k]},x_s^t),  \nonumber \\
        a_{i}^{t+1} &= a_{i}^{t} - \xi_{t} \frac{1}{B} \sum_{s=1}^B \nabla_{a_{i}^{t}} L( \theta^{t}, \chi_{[k]},x_s^t) +c_t, \label{eq:SGD_updates}\\
        b_{i}^{t+1} &= \lambda_t \left( b_{i}^{t} + \psi_{t}  \frac{1}{B} \sum_{s=1}^B \nabla_{b_{i}^{t}} L( \theta^{t}, \chi_{[k]},x_s^t) \right) + d_t,\nonumber
\end{align}
where $L(\theta^t, \chi_{[k]}, x) = \max \{ 0,1- \chi_{[k]}(x) \NN(x;\theta^t) \}$. Following the 2-steps curriculum strategy introduced above, we set
\begin{align}
    & x_s^0 \overset{iid}{\sim} \Rad\left(p\right)^{\otimes d} \qquad &\forall s \in [B], \qquad \\
    & x_s^t \overset{iid}{\sim} \Rad\left(1/2 \right)^{\otimes d} \quad \qquad  &\forall t \geq 1, s \in [B], \qquad
\end{align}
where $p =\frac{1}{2} \sqrt{1- \frac{1}{2(d-k)} } +\frac{1}{2} $. For brevity, we denote $\mu := 2p-1 = \sqrt{1-\frac{1}{2(d-k)}}$.
We set the parameters of SGD to:
\begin{align}
        & \gamma_0 = \mu^{-(k-1)}2N,   &\gamma_t = 0 \qquad \forall t\geq 1, \qquad \\
        & \xi_0 = 0,  &\xi_t = \frac{\epsilon}{2N} \qquad  \forall t \geq 1, \qquad  \\
        & \psi_0 = \frac{N}{\mu^k}, & \psi_t = 0  \qquad \forall t \geq 1, \qquad\\
        & c_0 = -\frac{1}{2N}, & c_t = 0  \qquad \forall t \geq 1, \qquad\\
        & \lambda_0 = (d+1) , & \lambda_t = 1  \qquad \forall t \geq 1, \qquad\\
        & d_0 = 0, & d_t = 0  \qquad \forall t \geq 1,  \qquad
\end{align}
and we consider the following initialization scheme:
\begin{align}
    w_{i,j}^{(0)} &=0 \qquad \forall i \in [N], j \in [d]; \nonumber\\
    a_i^{(0)} & = \frac{1}{2N}  \qquad \forall i \in [N]; \label{eq:init}\\
    b_i^{(0)} &\sim \Unif \Big\{ \frac{b_{lm}}{d+1}+\frac{1}{2} : l \in \{ 0,...,d\}, m \in \{ -1,...,d-k\} \Big\}, \nonumber
\end{align}
where we define 
\begin{align} \label{eq:blm}
        b_{lm} := -d +2l -\frac{1}{2} +\frac{m+1}{d-k}.
\end{align}
Note that such initialization is invariant to permutations of the input neurons. We choose such initialization because it is convenient for our proof technique. We believe that the argument may generalize to more standard initialization (e.g. uniform, Gaussian), however this would require more work and it may not be a trivial extension. 
\subsection{First Step: Recovering the Support}
As mentioned above, we train our network for one step on $(x_i, \chi_{[k]}(x_i))_{i \in [B]}$ with $x_i \sim \Rad(p)^{\otimes d}$. 
\paragraph{Population gradient at initialization.} Let us compute the population gradient at initialization. Since we set $\xi_0 =0$, we do not need to compute the initial gradient for $a$. Note that at initialization $|\NN(x;\theta^0)| <1$. Thus, the initial population gradients are given by 
\begin{align}
    \forall j \in [k], i \in [N] \qquad G_{w_{i,j}} &= - a_i\E_{x\sim \Rad(p)^{\otimes d}} \left[ \prod_{l \in [k]\setminus j} x_l \cdot \mathds{1} ( \langle w_i, x \rangle +b_i \in [0,1] ) \right]  \\
    \forall j \not\in [k], i \in [N] \qquad G_{w_{i,j}} &= - a_i\E_{x\sim \Rad(p)^{\otimes d}} \left[ \prod_{l \in [k]\cup j} x_l  \cdot \mathds{1} ( \langle w_i, x \rangle +b_i \in[ 0,1] ) \right] \\
    \forall i \in [N] \qquad  G_{b_i}& = -a_i \E_{x\sim \Rad(p)^{\otimes d}} \left[ \prod_{l \in [k]} x_l  \cdot \mathds{1} ( \langle w_i, x \rangle +b_i \in [0,1]) \right]
\end{align}


\begin{lemma}
Initialize $a,b,w$ according to~\eqref{eq:init}. Then,
\begin{align}
    \forall j \in [k], \qquad &G_{w_{i,j}}  = -\frac{ \mu^{k-1}}{2N} ;\\
    \forall j \not\in [k], \qquad & G_{w_{i,j}}  = -\frac{\mu^{k+1}}{2N};\\
     & G_{b_i}   = -\frac{\mu^{k}}{2N}.
\end{align}
\end{lemma}

\begin{proof} If we initialize according to~\eqref{eq:init}, we have $\langle w_i ,x \rangle +b_i \in [0,1]$ for all $i$. The results holds since $ \E_{x\sim \Rad(p)^{\otimes d} }[\chi_S(x) ] = \mu^{|S|}$.
\end{proof}

\paragraph{Effective gradient at initialization.}
\begin{lemma}
Let 
\begin{align}
        &\hat G_{w_{i,j}} : = \frac{1}{B} \sum_{s=1}^B \nabla_{w_{i,j}^{0}} L( \theta^{0}, \chi_{[k]},x_s^t)  \\ 
        & \hat G_{b_i} :=  \frac{1}{B} \sum_{s=1}^B \nabla_{b_{i}^{0}} L( \theta^{0}, \chi_{[k]},x_s^t)
\end{align}
be the effective gradients at initialization. If $B \geq (8\zeta^2 N^2)^{-1}\log(\frac{Nd+N}{\delta})$, then with probability $1-2\delta$,
\begin{align}
         \| \hat G_{w_{i,j}} - G_{w_{i,j}} \|_\infty & \leq \zeta , \\
         \| \hat G_{b_{i}} - G_{b_{i}} \|_\infty  & \leq \zeta ,
    \end{align}
    where $G_{w_{i,j}},G_{b_{i}}$ are the population gradients.
\end{lemma}

\begin{proof}
We note that $\E[\hat G_{w_{i,j}} ] =G_{w_{i,j}}  $, $\E[\hat G_{b_{i}} ] =G_{b_{i}}  $, and $|\hat G_{w_{i,j}}|,|\hat G_{b_{i}}| \leq \frac{1}{2N} $
\begin{align}
        & \pr_{x \sim \Rad(p)^{\otimes d}} \left(  \mid \hat G_{w_{i,j}}  - G_{w_{i,j}}  \mid \geq \zeta  \right) \leq 2 \exp \left(  - 8\zeta^2 N^2 B  \right) \leq \frac{2 \delta}{Nd+N},\\
        & \pr_{x \sim \Rad(p)^{\otimes d}} \left(  \mid \hat G_{b_{i}}  - G_{b_{i}}  \mid \geq \zeta  \right) \leq 2 \exp \left(  - 8\zeta^2 N^2 B \right) \leq  \frac{2 \delta}{Nd+N}.
\end{align}
The result follows by union bound.
\end{proof} 

\begin{lemma} \label{lem:bound_estimated_gradient}
Let
    \begin{align}
        & w_{i,j}^{(1)} = w_{i,j}^{(0)} - \gamma_0 \hat G_{w_{i,j}}\\
        & b_{i}^{(1)} = \lambda_0 \left( b_{i}^{(0)} - \psi_0 \hat G_{b_{i}} \right)\\
    \end{align}
If $B \geq (8\zeta^2 N^2)^{-1}\log(\frac{Nd+N}{\delta})$, with probability $1-2\delta$
\begin{enumerate}
    \item[i)] For all $j \in [k] $, $i \in [N]$, $|  w_{i,j}^{(1)} - 1|  \leq  \frac{2 N  \zeta}{\mu^{k-1}} $;
    \item[ii)] For all $j \not\in [k]$, $| w_{i,j}^{(1)} - (1-\frac{1}{2(d-k)})| \leq  \frac{2 N  \zeta}{\mu^{k-1}} $ ;
    \item[iii)] For all $i \in [N]$, $| b_i^{(1)} -  (d+1)( b_i^{(0)} -\frac{1}{2}) | \leq \frac{N(d+1) \zeta }{\mu^k}$.
\end{enumerate}
    
\end{lemma}
\begin{proof}
We apply Lemma~\ref{lem:bound_estimated_gradient}:
\begin{enumerate}
    \item[i)] For all $j \in [k] $, $i \in [N]$, $| \hat w_{i,j}^{(1)} - 1| = \gamma_0 |\hat G_{w_{i,j}} - G_{w_{i,j}} | \leq  \frac{2 N  \zeta}{\mu^{k-1}}  $;
    \item[ii)] For all $j \not\in [k]$,
    $i \in [N]$, $| \hat w_{i,j} - (1-\frac{1}{2(d-k)})| = \gamma_0 |\hat G_{w_{i,j}} - G_{w_{i,j}} | \leq  \frac{2 N  \zeta}{\mu^{k-1}}  $;
    \item[iii)] For all $i \in [N]$, 
    \begin{align}
    | \hat b_i^{(1)} -  (d+1)( b_i^{(0)} -\frac{1}{2}) |& =  | \lambda_0 (b_i^{(0)} + \psi_0 \hat G_{b_i} ) - \lambda_0 (b_i^{(0)} + \psi_0 G_{b_i}) | \\
    &\leq |\lambda_0 | \cdot |\psi_0| \cdot |\hat G_{b_i} - G_{b_i} | \\
    &\leq \frac{N(d+1) \zeta }{\mu^k}.
    \end{align}
\end{enumerate}
\end{proof}

\begin{lemma} \label{lem:hidden_layer}
If $N \geq (d+1)(d-k+1) \log((d+1)(d-k+1)/\delta) $, then with probability $1-\delta$, for all $l \in \{ 0,...,d\}$, and for all $m \in \{ -1,..., d-k\}$ there exists $i $ such that $b_i^{(0)} = \frac{b_{lm}}{d+1}+\frac{1}{2}$.
\end{lemma}

\begin{proof}
The probability that there exist $l,m$ such that the above does not hold is
\begin{align}
    \left(1- \frac{1}{(d+1) (d-k+1)} \right)^N \leq \exp \left(  - \frac{N}{(d+1) (d-k+1)}\right)  \leq \frac{\delta}{(d+1)(d-k+1)}.
\end{align}
The result follows by union bound.
\end{proof}

\begin{lemma} \label{lem:bound_sigma_lm}
Let $\sigma_{lm}(x) = \Ramp \left(\sum_{j=1}^d x_j - \frac{1}{2(d-k)} \sum_{j>k}  x_j + b_{lm} \right)  $, with $b_{lm}$ given in~\eqref{eq:blm}. 
If $B \geq (8\zeta^2 N^2)^{-1}\log(\frac{Nd+N}{\delta})$ and $N \geq (d+1)(d-k+1) \log((d+1)(d-k+1)/\delta) $, with probability $1-3 \delta$, for all $l,m$ there exists $i$ such that 
\begin{align}
    \Big|\sigma_{lm}(x) - \Ramp \left( \sum_{j=1}^d \hat w_{i,j}^{(1)} x_j + \hat b_{i}^{(1)}\right) \Big| \leq  3 N (d+1) \zeta \mu^{-k}. 
\end{align}
\end{lemma}
\begin{proof}
By Lemma~\ref{lem:hidden_layer}, with probability $1-\delta$, for all $l,m $ there exists $i$ such that $b_i^{(0)} = \frac{b_{lm}}{d+1}+\frac{1}{2}$. 
For ease of notation, we replace indices $i \mapsto (lm)$, and denote $\hat \sigma_{lm}(x) = \Ramp \left( \sum_{j=1}^d  w_{lm,j}^{(1)} x_j + b_{lm}^{(1)}\right)$. Then, by Lemma~\ref{lem:bound_estimated_gradient} with probability $1-2\delta$,
\begin{align}
    | \sigma_{lm}(x) - \hat \sigma_{lm} (x) | &\leq \Big| \sum_{j=1}^k ( w_{lm,j}^{(1)} - 1) x_j  + \sum_{j=k+1}^{d} \left(  w_{lm,j}^{(1)} - \left(1-\frac{1}{2(d-k)}\right) \right) x_j  + b_{lm}^{(1)} - b_{lm} \Big| \\
    & \leq k 2 N\zeta \mu^{-(k-1)}  + (d-k) 2 N \zeta \mu^{-(k-1)} + N(d+1) \zeta \mu^{k}\\
    & \leq 3 N (d+1) \zeta \mu^{-k}.
\end{align}

\end{proof}

\begin{lemma} \label{lem:span}
    There exists $a^*$ with $\| a^* \|_{\infty} \leq 4 (d-k)$ such that 
    \begin{align}
        \sum_{l=0}^d \sum_{m=-1}^{d-k} a_{lm}^* \sigma_{lm}(x) = \chi_{[k]}(x).
    \end{align}
\end{lemma}

\begin{proof}
Recall, that we assumed $d,k$ even and recall that
\begin{align}
\sigma_{lm}(x) = \Ramp \left( \sum_{j=1}^d x_j - \frac{1}{2(d-k)} \sum_{j>k}  x_j + b_{lm} \right) ,
\end{align}
where $b_{lm} = -d+2l - \frac{1}{2}+ \frac{m+1}{d-k} $ for $ l \in [d], m \in \{-1,...,d-k+1\} $ and 
$\Ramp(x) =  \begin{cases}
      0 \quad x \leq 0,\\
      x \quad 0<x \leq 1,\\
      1 \quad x >1
    \end{cases}\,.$\\

Let $ \sum_{j=1}^d x_j = d-2 t$, where $t$ is the total number of $-1$, and similarly let $ -\sum_{j=k+1}^d x_j = (d-k) -2 s $, where $s$ is the number of $+1$ outside the support of the parity $\chi_{[k]}(x)$.We have, 
\begin{align}
        \sigma_{lm}(x) = \Ramp \left( 2(l-t) + \frac{m+1-s}{d-k}   \right).
\end{align}
We take 
\begin{align}
        & a_{lm}^* = (-1)^l (-1)^{m} 2 (d-k) \qquad \forall l \in [d], m= -1,\\
        &a_{lm}^* = (-1)^l  (-1)^{m} 4 (d-k) \qquad \forall l \in [d], m \in \{0, 1,...,d-k-2\},\\
        & a_{lm}^* = (-1)^l  (-1)^{m} 3 (d-k) \qquad \forall l \in [d], m= d-k-1,\\
        & a_{lm}^* = (-1)^l  (-1)^{m} (d-k)\qquad \forall l \in [d], m= d-k,
\end{align}
Note that for all $l < t$, 
\begin{align}
   2( l-t )+ \frac{m+1-s}{d-k} \leq -2 + \frac{d-k+1}{d-k} \leq 0,
\end{align}
thus, $\sigma_{lm}(x) =0$ for all $m$. Moreover, for all $l>t$,
\begin{align}
    2(l-t) + \frac{m-s+1}{d-k}  \geq 2- \frac{d-k}{d-k} = 1.
\end{align}
Thus, $\sigma_{lm}(x) =1$ for all $m$ and
\begin{align}
        \sum_{m=-1}^{d-k} a_{lm}^* \sigma_{lm}(x) = \sum_{m=-1}^{d-k} a_{lm}^* =0.
\end{align}
If $l=t$,
\begin{align*}
    \sum_{m=-1}^{d-k}& a_{tm}^* \sigma_{tm}(x) \\
    & = 
 (-1)^t (d-k)\left[ \sum_{m=0}^{d-k-2} 4 (-1)^{m} \Ramp \left( \frac{m+1-s}{d-k} \right)  - 3  \Ramp\left(\frac{d-k-s}{d-k} \right) + \Ramp \left( \frac{d-k+1-s }{d-k}\right) \right] \\
 & = (-1)^t \left[ \sum_{m=s}^{d-k-2} 4 (-1)^{m}\left( m+1-s \right)_+  - 3  \left(d-k-s\right)_+ +  \left( d-k+1-s \right)_+ \right] \\
 & = (-1)^t (-1)^s.
\end{align*}
Since we assumed $d,k$ even, $(-1)^s = \prod_{i=k+1}^{d} x_i$. Moreover, observe that $ \chi_{[k]}(x) = \prod_{i=k+1}^d x_i \cdot \prod_{i=1}^d x_i$.
Thus, 
\begin{align}
\sum_{lm} a_{lm}^* \sigma_{lm}(x) =(-1)^t (-1)^s = \chi_{[k]}(x).
\end{align}

\end{proof}

\noindent 
\begin{lemma}
Let $f^*(x) = \sum_{l,m} a_{lm}^* \sigma_{lm}(x) $ and let $\hat f(x) = \sum_{l,m} a^*_{lm} \hat \sigma_{lm}(x) $, with $\sigma_{lm}, \hat \sigma_{lm}$ defined in Lemma~\ref{lem:bound_sigma_lm} and $a^* $ defined in Lemma~\ref{lem:span}. If $B \geq (8\zeta^2 N^2)^{-1}\log(\frac{Nd+N}{\delta})$ and $N \geq (d+1)(d-k+1) \log((d+1)(d-k+1)/\delta) $, with probability $1-3 \delta$ for all $x$,
\begin{align}
        L(\hat f,f^*,x)  \leq (d+1)^2 (d-k+1)^2 12 N \zeta \mu^{-k}.
\end{align}

\end{lemma}

\begin{proof}
\begin{align}
 |f^*(x) -\hat f(x) | &= \Big| \sum_{l,m} a_{l,m}^* (\sigma_{lm}(x) - \hat \sigma_{lm}(x) \Big| \\
 & \leq d(d-k+1) \|a^*\|_{\infty }  \sup_{lm} |\sigma_{lm}(x) - \hat \sigma_{lm}(x) | \\
 & \leq (d+1)^2 (d-k+1)^2 12 N \zeta \mu^{-k}.
\end{align}

Thus,
\begin{align}
    (  1 -f(x) f^*(x) )_+ & \leq |1 -f(x) f^*(x) |\\
    & = | f^{*^2}(x) - f(x) f^*(x) | \\
    & = |f^*(x) | \cdot | f^*(x) - f(x) | \leq(d+1)^2 (d-k+1)^2 12 N \zeta \mu^{-k},
\end{align}
which implies the result.
\end{proof}

\subsection{Second Step: Convergence}
To conclude, we use an established result on convergence of SGD on convex losses (see e.g.~\cite{shalev2012online,shalev2014understanding,daniely2020learning,malach2020computational,barak2022hidden}).
\begin{theorem} \label{thm:convergence_SGD}
    Let $\cL$ be a convex function and let $a^* \in \argmin_{\|a\|_2 \leq \cB} \cL(a) $, for some $\cB>0$. For all $t$, let $\alpha^{(t)}$ be such that $\E\left[ \alpha^{(t)} \mid a^{(t)} \right] =  -\nabla_{a^{(t)}} \cL(a^{(t)}) $ and assume $\| \alpha^{(t)} \|_2 \leq \rho $ for some $\rho>0$. If $a^{(0)} = 0 $ and for all $t \in [T]$ $a^{(t+1)} = a^{(t)} +\gamma \alpha^{(t)} $, with $\gamma = \frac{\cB}{\rho \sqrt{T}}$, then :
    \begin{align}
        \frac{1}{T} \sum_{t=1}^T \cL(a^{(t)}) \leq \cL(a^*) + \frac{\cB \rho}{\sqrt{T}}.
    \end{align}
\end{theorem}

Let $\cL(a): = \E_{x \sim \Rad(1/2)^{\otimes d}} \left[ L((a,b^{(1)},w^{(1)} ), \chi_{[k]}, x) \right]$. Then, $\cL$ is convex in $a$ and for all $t \in [T]$,
\begin{align}
        \alpha^{(t)} &= - \frac{1}{B} \sum_{s=1}^B \nabla_{a^{(t)}} L((a^{(t)},b^{(1)},w^{(1)} ),\chi_{[k]}, x) \\
        & = - \frac{1}{B} \sum_{s=1}^B  \sigma( w^{(1)}x + b^{(1)}) .
\end{align}
Thus, recalling $\sigma = \Ramp$, we have $\| \alpha^{(t)} \|_2 \leq  \sqrt{N}$.
Let $a^*$ be as in Lemma~\ref{lem:span}. Clearly, $\| a^* \|_2 \leq 4 (d-k+1)^{3/2}(d+1)^{1/2}.$ Moreover, $a^{(1)} =0$. Thus, we can apply Theorem~\ref{thm:convergence_SGD} with $\cB = 4 (d-k+1)^{3/2}(d+1)^{1/2}$, $ \rho = \sqrt{N} $ and obtain that if 
\begin{enumerate}
    \item $N \geq (d+1)(d-k+1) \log((d+1)(d-k+1)/\delta)$;
    \item $\zeta \leq \frac{\epsilon \mu^k}{24 (d+1)^2 (d-k+1)^2 N }$;
    \item $ B \geq (8\zeta^2 N^2)^{-1}\log(\frac{Nd+N}{\delta})$;
    \item $T \geq \frac{64}{\epsilon^2} (d-k+1)^3 (d+1) N$.
\end{enumerate}
then, with probability $1-3 \delta $ over the initialization
\begin{align}
    \E_{x \sim \Rad(1/2)^{\otimes d}} \left[ \min_{t \in [T]} L \left( \theta^{(t)}, \chi_{[k]}, x \right)  \right] \leq \frac{\epsilon}{2} + \frac{\epsilon}{2} = \epsilon.
\end{align}
\begin{remark}
We assume $k \leq d/2$ to avoid exponential dependence of $\zeta$ (and consequently of the batch size and of the computational complexity) in $d$. Indeed, if $k\leq d/2$,
then, 
\begin{align}
        \mu^k = \left( 1- \frac{1}{2(d-k)}\right)^{k/2} \geq \left( 1- \frac{1}{d} \right)^{d/4} \sim e^{-1/4}.
\end{align}
\end{remark}

\section{Proof of Theorem~\ref{thm:positive_cov}}
\label{sec:proof_positive_cov}
\begin{theorem}[Theorem~\ref{thm:positive_cov}, restatement]
        Let $k,d$ be integers, such that $d\geq k$ and $k$ even. Let $\NN(x;\theta) = \sum_{i=1}^N a_i \sigma(w_ix + b_i) $ be a 2-layers fully connected network with activation $\sigma(y) := \ReLU(y) $ and $N \geq (k+1) \log(\frac{k+1}{\delta})$. Consider training $\NN(x;\theta) $ with SGD on the covariance loss with batch size $B \geq (2\zeta^2)^{-1} \log(\frac{dN}{\delta})$, with $\zeta \leq \frac{\epsilon (\mu^{k-1}-\mu^{k+1})}{64k^2N} \cdot \left( 1+ \frac{d-k}{k}\right)^{-1}$, for some $\mu \in (0,1)$. Then, there exists an initialization, a learning rate schedule, and a 2-CL strategy such that after $T \geq \frac{64k^3 N}{\epsilon^2}$ iterations, with probability $1-3\delta$ SGD outputs a network with generalization error at most $\epsilon$.
\end{theorem}


\subsection{Proof Setup}
Similarly as before, we consider a 2-layers neural network, defined as $\NN(x; \theta ) = \sum_{i=1}^N a_i \sigma (w_i x +b_i ),$ where $N$ is the number of hidden units, $\theta = (a,b,w)$ and $\sigma := \ReLU$. 
Our proof scheme is similar to the previous Section. Again, we assume without loss of generality that the labels are generated by $\chi_{[k]}(x):= \prod_{i=1}^k x_i$. We assume $k$ to be even.
We train our network with SGD on the covariance loss, defined in Def.~\ref{def:covariance_loss}. We use the same updates as in~\eqref{eq:SGD_updates} with: 
\begin{align}
    &x_s^0 \overset{iid}{\sim} \Rad(p)^{\otimes d} \qquad \forall s \in [B],\\
    &x_s^t \overset{iid}{\sim} \Rad(p)^{\otimes d} \qquad \forall t\geq 1, s \in [B],
\end{align}
for some $p \in (1/2,1)$. We denote $\mu := 2p-1$. We set the parameters to:
\begin{align}
        & \gamma_0 = 16N (\mu^{k-1}-\mu^{k+1})^{-1} k^{-1},   &\gamma_t = 0 \qquad \forall t\geq 1, \qquad \\
        & \xi_0 = 0,  &\xi_t = \frac{\epsilon}{8N} \quad  \forall t \geq 1, \qquad  \\
        & \psi_0 =0, & \psi_t = 0  \qquad \forall t \geq 1, \qquad\\
        & c_0 = -1, & c_t = 0  \qquad \forall t \geq 1, \qquad\\
        & \lambda_0 = 1 , & \lambda_t = 1  \qquad \forall t \geq 1, \qquad\\
        & d_0 = -1, & d_t = 0  \qquad \forall t \geq 1,  \qquad
\end{align}
and we consider the following initialization scheme:
\begin{align}
    w_{i,j}^{(0)} &=0 \qquad \forall i \in [N], j \in [d]; \nonumber\\
    a_i^{(0)} & = \frac{1}{16N}  \qquad \forall i \in [N]; \label{eq:init_2}\\
    b_i^{(0)} &\sim \Unif \Big\{\frac{2(i+1)}{k} : i \in \{ 0,1,...,k\}\Big\}. \nonumber
\end{align}

\subsection{First Step: Recovering the Support}
\paragraph{Population gradient at initialization.} 
At initialization, we have $|\NN(x ; \theta^0)| <\frac{1}{4}$, thus
\begin{align}
       \Big | \Cov(\chi_{[k]}, \theta^0, x, \Rad(p)^{\otimes d} ) \Big| & < 1.
\end{align}
The initial gradients are therefore given by:
\begin{align*}
    \forall i,j \qquad G_{w_{i,j}} &= - a_i\E_{x\sim \Rad(p)^{\otimes d}} \left[ \Big(\prod_{l \in [k]} x_l - \mu^{k}\Big) \cdot \Big( x_j \mathds{1} ( \langle w_i, x \rangle +b_i >0 ) - \E x_j \mathds{1} ( \langle w_i, x \rangle +b_i >0 )\Big) \right]  \\
    \forall i \in [N] \qquad  G_{b_i}& = - a_i\E_{x\sim \Rad(p)^{\otimes d}} \left[ \Big(\prod_{l \in [k]} x_l - \mu^{k}\Big) \cdot \Big( \mathds{1} ( \langle w_i, x \rangle +b_i >0 ) - \E \mathds{1} ( \langle w_i, x \rangle +b_i >0 )\Big) \right] 
\end{align*}

If we initialize $a,b,w$ according to~\eqref{eq:init_2}. Then,
\begin{align}
    \forall j \in [k], \qquad &G_{w_{i,j}}  = -\frac{\mu^{k-1} -\mu^{k+1} }{16 N};\\
    \forall j \not\in [k], \qquad & G_{w_{i,j}}  = 0;\\
     & G_{b_i}   =0.
\end{align}

\paragraph{Effective gradient at initialization.}

\begin{lemma} \label{lem:bound_estimated_gradient_2}
Let
    \begin{align}
        &w_{i,j}^{(1)} = w_{i,j}^{(0)} - \gamma_0 \hat G_{w_{i,j}}\\
        & b_{i}^{(1)} = \lambda_0 \left( b_{i}^{(0)} - \psi_0 \hat G_{b_{i}} \right)+d_0,\\
    \end{align}
    where $\hat G_{w_{i,j}}, \hat G_{b_i} $ are the gradients estimated from the initial batch.
Then, with probability $1-2\delta$, if $B \geq (2 \zeta^2)^{-1} \log\left(\frac{dN}{\delta} \right)$,
\begin{enumerate}
    \item[i)] For all $j \in [k] $, $i \in [N]$, $| w_{i,j}^{(1)} - \frac{1}{k}|  \leq  \frac{\zeta 16N }{k(\mu^{k-1}-\mu^{k+1})} $;
    \item[ii)] For all $j \not\in [k]$, $|w_{i,j}^{(1)} | \leq  \frac{\zeta 16N }{k(\mu^{k-1}-\mu^{k+1})}$ ;
    \item[iii)] For all $i \in [N]$, $  b_i^{(1)} =b^{(0)} -1 $
\end{enumerate}
    
\end{lemma}
\begin{proof}
By Lemma~\ref{lem:bound_estimated_gradient} ,if $B \geq (2 \zeta^2)^{-1} \log\left(\frac{dk}{\delta} \right)$, then for all $j \in [k], i \in [N]$, $| \hat G_{w_{i,j}} - G_{w_{i,j}} | \leq \zeta$. Thus, 
\begin{enumerate}
    \item[i)] For all $j \in [k] $, $i \in [N]$, $|  w_{i,j}^{(1)} - \frac{1}{k}| = \gamma_0 |\hat G_{w_{i,j}} - G_{w_{i,j}} | \leq\frac{\zeta 16 N }{k(\mu^{k-1}-\mu^{k+1})} $;
    \item[ii)] For all $j \in [k] $, $i \in [N]$, $|  w_{i,j}^{(1)} | = \gamma_0 |\hat G_{w_{i,j}} - G_{w_{i,j}} | \leq\frac{\zeta 16N }{k(\mu^{k-1}-\mu^{k+1})} $;
\end{enumerate}
iii) follows trivially.
\end{proof}

\noindent 

\begin{lemma} \label{lem:hidden_layer_2}
If $N \geq (k+1) \log \left(\frac{k+1}{\delta} \right)$, with probability $1-\delta$ for all $i \in \{0,...,k\} $ there exists $l \in [N]$ such that $ b_l^{(0)} = \frac{2(i+1)}{k} $.
\end{lemma}

\begin{proof}
The probability that there exists $i$ such that the above does not hold is
\begin{align}
    \left(1- \frac{1}{k+1} \right)^N \leq \exp \left(  - \frac{N}{k+1}\right)  \leq \frac{\delta}{k+1}.
\end{align}
The result follows by union bound.
\end{proof}

\begin{lemma} 
Let $\sigma_{i}(x): = \ReLU \left(\frac{1}{k} \sum_{j=1}^k x_j + b_{i} \right)  $, with $b_{i} = -1+\frac{2(i+1)}{k}$. 
Then, with probability $1-3 \delta$, for all $i \in \{ 0,...,k\}$ there exists $l\in [N]$ such that 
\begin{align}
    \Big|\sigma_{i}(x) - \ReLU \left( \sum_{j=1}^d w_{l,j}^{(1)} x_j +  b_{l}^{(1)}\right) \Big| \leq \frac{\zeta 16 N  }{\mu^{k-1}-\mu^{k+1}} \cdot \left( 1+ \frac{d-k}{k} \right).
\end{align}
\end{lemma}
\begin{proof}
By Lemma~\ref{lem:hidden_layer_2} and Lemma~\ref{lem:bound_estimated_gradient_2},  
with probability $1-3\delta$, for all $i$ there exists $l$ such that $ b^{(1)}_l = -1 + \frac{2(i+1)}{k} $, and 
\begin{align}
        \Big|\sigma_i(x) - \ReLU \left( \sum_{j=1}^d w_{l,j}^{(1)} x_j +  b_{l}^{(1)}\right) \Big| & \leq \Big|  \sum_{j=1}^k  \left(\frac{1}{k} -  w_{l,j}^{(1)}  x_j \right) \Big| + \Big| \sum_{j=k+1}^d w_{l,j}^{(1)} x_{j} \Big| \\
        & \leq \frac{\zeta 16N }{\mu^{k-1}-\mu^{k+1}} + \frac{\zeta 16N (d-k) }{(\mu^{k-1}-\mu^{k+1})k} \\
        & =  \frac{\zeta 16N  }{\mu^{k-1}-\mu^{k+1}} \cdot \left( 1+ \frac{d-k}{k} \right).
\end{align}

\end{proof}

\begin{lemma} \label{lem:span_2}
    There exist $a^*$ with $\| a^* \|_{\infty} \leq 2k $ such that 
    \begin{align}
        \sum_{i=0}^k a_i^* \sigma_{i}(x) = \chi_{[k]}(x).
    \end{align}
\end{lemma}
\begin{proof}
We assume $k$ to be even. Let $\sum_{j=1}^k x_j = k-2t$, where $t:=|\{ i: x_i = -1, i \in [k]\}|$. Thus,
\begin{align}
        \sigma_i(x) = \ReLU \left(  \frac{2(i+1-t)}{k}\right).
\end{align}
We choose
\begin{align}
       & a_i^* = (-1)^i 2k \qquad \forall i \in \{ 0,1,...,k-2\},\\
        & a_i^* = (-1)^i \frac{3}{2}k \qquad i=k-1,\\
        & a_i^* = (-1)^i \frac{1}{2}k\qquad i =k.
\end{align}
One can check that with these $a^*_i$ the statement holds. 
\end{proof}

\noindent 
\begin{lemma}
Let $f^*(x) = \sum_{i=0}^k a_{i}^* \sigma_{i}(x) $ and let $\hat f(x) = \sum_{i=0}^k a^*_{i} \hat \sigma_{i}(x) $, with $\sigma_{i}(x) $ defined above and $ \hat \sigma_{i}(x) := \ReLU ( \sum_{j=1}^d w_{i,j}^{(1)}x_j +b_i^{(1)}) $. Then, with probability $1-3 \delta$ for all $x$,
\begin{align}
        (  1 -f(x) f^*(x) )_+\leq \frac{ 32 k^2 \zeta N  }{\mu^{k-1}-\mu^{k+1}} \cdot \left( 1+ \frac{d-k}{k} \right),
\end{align}
where $(z)_+ : = \max \{ 0, z \}$.
\end{lemma}

\begin{proof}
\begin{align}
 |f^*(x) -\hat f(x) | &= \Big| \sum_{i} a_{i}^* (\sigma_{i}(x) - \hat \sigma_{i}(x) \Big| \\
 & \leq k\|a^*\|_{\infty }  \sup_{i} |\sigma_{i}(x) - \hat \sigma_{i}(x) | \\
 & \leq \frac{ 32 k^2 \zeta N  }{\mu^{k-1}-\mu^{k+1}} \cdot \left( 1+ \frac{d-k}{k} \right).
\end{align}
Thus,
\begin{align}
    (  1 -f(x) f^*(x) )_+ & \leq |1 -f(x) f^*(x) |\\
    & = | f^{*^2}(x) - f(x) f^*(x) | \\
    & = |f^*(x) | \cdot | f^*(x) - f(x) | \leq \frac{ 32 k^2 \zeta N  }{\mu^{k-1}-\mu^{k+1}} \cdot \left( 1+ \frac{d-k}{k} \right).
\end{align}
\end{proof}

\subsection{Second Step: Convergence}
We apply Theorem~\ref{thm:convergence_SGD} with 
$\cL(a): = \E_{x \sim \Rad(1/2)^{\otimes d}} \left[ L_{\Cov} ((a,b^{(1)},w^{(1)} ), \chi_{[k]}, x) \right]$, $ \rho = 2\sqrt{N}$, $\cB = 2k \sqrt{k}$. We get that if
\begin{enumerate}
    \item $ N \geq (k+1) \log(\frac{k+1}{\delta})$;
    \item $ \zeta \leq \frac{\epsilon (\mu^{k-1}-\mu^{k+1})}{64k^2N} \cdot \left( 1+ \frac{d-k}{k}\right)^{-1}$;
    \item $ B \geq (2\zeta^2)^{-1} \log(\frac{dN}{\delta})  $;
    \item $T \geq  \frac{64k^3 N}{\epsilon^2}  $ .
\end{enumerate}
then with probability $1-3\delta$ over the initialization,

\begin{align}
    \E_{x \sim \Rad(1/2)^{\otimes d}} \left[ \min_{t \in [T]} L_{\Cov} \left( \chi_{[k]}, \theta^t, x, \Rad(1/2)^{\otimes d}  \right) \right]  \leq \epsilon.
\end{align}

\begin{remark}
We remark that if $ \mu = \theta_{d,k}(1)$, then $\zeta $ decreases exponentially fast in $k$, and as a consequence the batch size and the computational cost grow exponentially in $k$. If however we choose $\mu = 1-1/k$, then we get $\zeta = 1/\poly(k)$ and, as a consequence, the batch size and the computational cost grow polynomially in $k$.
\end{remark}

\section{Proof of Theorem~\ref{thm:negative_result}}
\label{sec:proof_negative}
Let us consider $r=2$. The case of general $r$ follows easily. Let us state the following Lemma.

        \begin{lemma} \label{lem:hamming_Hoeffding}
        Let $x \sim \Rad(p)^{\otimes d}$ and let $H(x) :=\sum_{i=1}^d \mathds{1} (x_i =1)$ be the Hamming weight of $H(x)$. Assume $\epsilon <p < \epsilon'$ for some $\epsilon, \epsilon' \in [0,1/2]$, then, 
        \begin{align}
            &\pr_{x \sim \Rad(p)^{\otimes d}} \left( H(x) \geq  \epsilon' d \right)  \leq 2 \exp \left(- (\epsilon'-p)^2 d\right); \\
            &\pr_{x \sim \Rad(p)^{\otimes d}} \left( H(x) \leq  \epsilon d \right) \leq 2 \exp \left(- (p-\epsilon)^2 d\right) .
        \end{align}
        \end{lemma}

        \begin{proof}[Proof of Lemma~\ref{lem:hamming_Hoeffding}]
        We apply Hoeffding's inequality with $\E[ H(x)] = pd$ and $\sum_{i=1}^d x_i = d- 2H(x)$:
        \begin{align}
            \pr(| H(x) - pd | \geq t d ) \leq 2 \exp \left( -(t-p)^2 d  \right).
        \end{align}
        \end{proof}

        \noindent 
        Take $\epsilon$ such that $ |p_1 -\frac{1}{2}| >\epsilon $, and consider the following algorithm:
        \begin{enumerate}
            \item Choose a set $V \subseteq[d] $ uniformly at random among all subsets of $[d]$ of cardinality $k_S$;
            \item Take a fully connected neural network $\NN(x;\psi)$ with the same architecture as $\NN(x;\theta)$ and with initialization $\psi^0 = \theta^0$;
            \item Train $\NN(x; \psi)$ on data $(x, \chi_V(x))$, with $x \sim \Rad(p_1)^{\otimes d}$;
            \item Train until convergence the pre-trained network $\NN(x;\psi)$ with initialization $\psi^{T_{1}}$ on data $(x, \chi_T(x))$, with $x \sim \Rad(1/2)^{\otimes d}$.
        \end{enumerate}
        The result holds by the following two Lemmas.
        \begin{lemma}\label{lem:TV} If $V =S$, 
            $ \TV(\theta^T;\psi^T) \leq \frac{A T \sqrt{|\theta| }}{\tau} \exp(-d \delta^2)$, where $\delta = \min\{ |\epsilon - p_1|, |1/2-\epsilon| \}$ and $\TV$ denotes the total variation distance between the law of $\theta^T$ and $\psi^T$.
        \end{lemma}
        \begin{proof}
        Clearly, $\TV(\theta^0;\psi^0) =0$. Then, using subadditivity of TV
        \begin{align}
            \TV (\theta^T;\psi^T)& \leq  \sum_{t=1}^T \TV(\theta^t ;\psi^t| \{Z^i \}_{i\leq t-2})\\
            & = \sum_{t=1}^T \TV( \gamma( g_{\theta^{t-1}} + Z^{t-1}); \gamma (g_{\psi^{t-1}}+ Z^{t-1})| \{Z_i \}_{i\leq t-1}),
        \end{align}
        where $g_{\theta^{t-1}},g_{\psi^{t-1}} $ denote the population gradients in $\theta^{t-1}$ and $\psi^{t-1} $, respectively. Then, recalling that the $Z^{t-1}$ are Gaussians, we get
        \begin{align}
            \TV (\theta^T;\psi^T)& \overset{(a)}{\leq} \sum_{t=1}^T \frac{1}{2\tau \gamma} \|  \gamma g_{\theta^{t-1}} - \gamma g_{\psi^{t-1}}   \|_2  \\
            & \overset{(b)}{\leq} \sum_{t=1}^{T_{r-1}}\frac{1 }{2\tau \gamma }  \cdot  2\sqrt{|\theta| }A \gamma \cdot  \pr( H(x) \geq  \epsilon d) + \sum_{t=T_{r-1}}^{T}\frac{1 }{2\tau \gamma  } \cdot  2\sqrt{|\theta| }A \gamma \cdot \pr( H(x) \leq  \epsilon d)\\
            & \overset{(c)}{\leq} \frac{A T \sqrt{|\theta| }}{\tau} \exp(-d \delta^2).
        \end{align}
        In $(a)$ we applied the formula for the TV between Gaussian variables with same variance. In $(b)$ we used that each gradient is in $[-A,A]$ and that during the first part of training, for all $x$ with $H(x) <\epsilon  $, the two gradients are the same, and similarly in the second part of training for all $H(x) > \epsilon d$. In $(c)$ we applied Lemma~\ref{lem:hamming_Hoeffding}.
        \end{proof}
        We apply Theorem 3 from~\cite{AS20}, which we restate here for completeness.

        \begin{theorem}[Theorem 3 in~\cite{AS20}]
            Let $\cP_k$ be the set of $k$-parities over $d$ bits. Noisy-GD on any neural network of size $|\theta|$ and any initialization, after $T$ steps of training under the uniform distribution, outputs a network such that 
            \begin{align}
                \frac{1}{|\cP_k|} \sum_{f \in \cP_k}| \E_{x\sim \Rad(1/2)^{\otimes d}} \left[ \NN(x; \theta^T) \cdot f(x) \right] | \leq \frac{ T \sqrt{|\theta|} A}{\tau} \cdot d^{-k/2}.
            \end{align}
        \end{theorem}
        In our case, this implies:
        \begin{align}
             \E_V| \E_x \left[ \NN(x;\psi^T) \cdot G_{S,T,\epsilon}(x)  \right]  | \leq  \frac{(T-T_{r-1})  \sqrt{|\theta|} A }{\tau d^{k_T/2}}        
             \end{align}

        To conclude our proof, note that:
        \begin{align*}
            \E_V  |\E_x \left[ \NN(x, \psi^T) \cdot G_{S,T,\epsilon}(x) \right] | & = \E_V\left[ \Big|\E_x \left[ \NN(x, \psi^T) \cdot G_{S,T,\epsilon}(x) \right] \Big| \mid \Big|V \cap T\Big|=0\right] \pr(|V\cap T |=0) \\
            & \quad + \E_V\left[ \Big|\E_x \left[ \NN(x, \psi^T) \cdot G_{S,T,\epsilon}(x) \right] \Big| \mid \Big|V \cap T\Big|>0\right]  \pr(|V\cap T |>0).
        \end{align*}
        One can check that,
\begin{align}
        \pr(|V\cap T |>0) =  1- \frac{2 k_S k_T}{d} + O(d^{-2}).
\end{align}
Moreover, by symmetry, for any $V$ such that  $|V \cap T |=0$, the algorithm achieves the same correlation (to see this, one find an appropriate permutation of the input neurons and use invariance of GD on fully connected networks with permutation-invariant initialization). Thus, 
\begin{align}
\E_V\left[ \Big|\E_x \left[ \NN(x, \psi^T) \cdot G_{S,T,\epsilon}(x) \right] \Big| \mid \Big|V \cap T\Big|=0\right]  =  \E_V\left[ \Big|\E_x \left[ \NN(x, \psi^T) \cdot G_{S,T,\epsilon}(x) \right] \Big| \mid V=S\right].
\end{align}
By Lemma~\ref{lem:TV}, 
\begin{align*}
    | \E_x \left[ \NN(x;\theta^T) \cdot G_{S,T,\epsilon}(x)  \right]  | &\leq  \E_V\left[ \Big|\E_x \left[ \NN(x, \psi^T) \cdot G_{S,T,\epsilon}(x) \right] \Big| \mid V=S\right] + \frac{AT \sqrt{|\theta|}}{\tau } \exp(-d \delta^2)\\
    &\leq \frac{(T-T_{r-1}) \sqrt{|\theta|} A}{\tau d^{k_T/2}} + \frac{AT \sqrt{|\theta|}}{\tau } \exp(-d \delta^2) + \frac{2 k_S k_T}{d} + O(d^{-2}).
\end{align*}



        The argument for general $r$ holds by taking $\epsilon $ such that $ |p_l -\frac{1}{2}| > \epsilon$ for all $l \in [r-1]$ and by replacing step $3$ of the algorithm above with the following:
        \begin{enumerate}
            \item[3.] Train $\NN(x;\psi)$ on data $(x, \chi_V(x))$ using a $(r-1)$-CL($(T_1,...,T_{r-1}),(p_1,...,p_{r-1})$) strategy.
        \end{enumerate}
        

\section{Proof of Corollary~\ref{cor:negative_nointersection}}
\label{sec:proof_cor_negative}
We use the same proof strategy of Theorem~\ref{thm:negative_result}: specifically, we use the same algorithm for training a network $\NN(x;\psi)$ with the same architecture as $\NN(x;\theta) $, so that Lemma~\ref{lem:TV} still holds. We import Theorem 3 from~\cite{AS20} in the following form. 
\begin{theorem}[Theorem 3 in~\cite{AS20}] \label{thm:AS20_cor}
    Let $\cF$ be the set of $k$-parities over set $\{d/2+1,...,d\}$. Noisy-GD on any neural network $\NN(x; w)$ of size $|w|$ and any initialization, after $T$ steps of training on samples drawn from the uniform distribution, outputs a network such that 
        \begin{align}
                \frac{1}{|\cF|} \sum_{f \in \cF} \Big| \E_{x\sim \Rad(1/2)^{\otimes d}} \left[ \NN(x; w^{(T)}) \cdot f(x) \right] \Big| \leq \frac{ T \sqrt{|w|} A}{\tau}  \cdot {d/2 \choose k}^{-1/2} .
        \end{align}
\end{theorem}
\noindent
Similarly as before, Theorem~\ref{thm:AS20_cor} and Lemma~\ref{lem:hamming_Hoeffding} imply:
        \begin{align}
            \E_V \Big| \E_{x\sim \Rad(1/2)^{\otimes d}} \left[ \NN(x;\psi^{(T)}) \cdot G_{S,V,\epsilon}(x) \right]  \Big| \leq  \frac{(T-T_{r-1})  \sqrt{|\theta|} A }{\tau }\cdot {d/2 \choose k_V}^{-1/2} + \exp(-d \delta^2),
        \end{align}
        where by $\E_V$ we denote the expectation over set $V$ sampled uniformly at random from all subsets of $\{d/2+1,...,d\}$ of cardinality $k_V$.
Since both the initialization and noisy-GD on fully connected networks are invariant to permutation of the input neurons, for any $V \subseteq \{d/2+1,...,d\}$, the algorithm achieves the same correlation. 
Thus, applying Lemma~\ref{lem:TV}:
\begin{align}
    \Big| \E_x \left[ \NN(x;\theta^{(T)}) \cdot G_{S,V,\epsilon}(x)  \right]  \Big|
    & \leq \frac{(T-T_{r-1}) \sqrt{|\theta|} A}{\tau }\cdot {d/2 \choose k_V}^{-1/2}  + \frac{2AT \sqrt{|\theta|}}{\tau } \exp(-d \delta^2).
\end{align}

\end{document}